\documentclass[12pt]{article}

\usepackage{graphicx}
\usepackage{epstopdf}
\usepackage{verbatim}
\usepackage{float}
\usepackage{amsmath,amssymb,mathtools} % define this before the line numbering.
\usepackage[ruled,vlined]{algorithm2e}
\usepackage{color}
\usepackage{geometry}
\usepackage{enumerate}
\usepackage{arydshln}
\usepackage{multirow}
\usepackage{rotating}
\usepackage{booktabs}
\usepackage[english]{babel}
\usepackage{caption}
\usepackage{mdwlist}
\usepackage{mathrsfs}
\usepackage[pagebackref=true,breaklinks=true,letterpaper=true,colorlinks,bookmarks=false]{hyperref}
\DeclareMathAlphabet{\mathpzc}{OT1}{pzc}{m}{it}
\DeclareMathAlphabet\mathbfcal{OMS}{cmsy}{b}{n}
\newcommand{\graph}{{\cal G}}
\newcommand{\nodes}{{\cal V}}
\newcommand{\edges}{{\cal E}}
\newcommand{\flow}{\digamma}

\newcommand{\labelset}{\mathcal{L}}
\newcommand{\labelvar}{f}
\newcommand{\labelvars}{f}
\newcommand{\unary}{D}

\newcommand{\RANSAC}{{\small\sc {RANSAC}}}
\newcommand{\PEARL} {{\small\sc {PEaRL}}}
\newcommand{\params}{\mathbf{\theta}}
\newcommand{\param}{\theta}
\newcommand{\subL}{S_l}
\newcommand{\subR}{S_r}
\newcommand{\rF}{\mathcal{F}_r}
\newcommand{\lF}{\mathcal{F}_l}
\newcommand{\pq}{(p,q)}
\newcommand{\M}{\mathcal{M}}
\newcommand{\N}{\mathcal{N}}

\newtheorem{theorem}{Theorem}[section]
\newtheorem{lemma}[theorem]{Lemma}

\newtheorem{corollary}[theorem]{Corollary}

\newenvironment{proof}[1][Proof]{\begin{trivlist}
\item[\hskip \labelsep {\bfseries #1}]}{\end{trivlist}}

\begin{document}

\pagestyle{myheadings}
\markright{\small H. Isack and Y. Boykov, arXiv:1303.2607v2, April 2014 \hfill}
\setcounter{page}{1}

\title{Joint optimization of {\em fitting \& matching}\\ in multi-view reconstruction}
%\author{First Author\\
%Institution1\\
%Institution1 address\\
%{\tt\small firstauthor@i1.org}}

\author{Hossam Isack \hspace{5ex}  Yuri Boykov \\
Computer Science Department \\
Western University, Canada \\
{\tt\small habdelka@csd.uwo.ca, \hspace{2ex} yuri@csd.uwo.ca}}

\date{\today}
\date{April 8, 2014}
\maketitle

\thispagestyle{myheadings}

\begin{abstract}
Many standard approaches for geometric model fitting are based on pre-matched image features.
Typically, such pre-matching uses only feature appearances (e.g.~SIFT) and a large number of non-unique features must be discarded in order to control the false positive rate.
In contrast, we solve feature matching and multi-model fitting problems in a joint optimization framework.
This paper proposes several {\em fit-\&-match} energy formulations based on a generalization of the {\it assignment problem}.
We developed an efficient solver based on {\it min-cost-max-flow} algorithm that finds near optimal solutions.
Our approach significantly increases the number of detected matches. In practice, energy-based joint
{\em fitting \& matching} allows to increase the distance between view-points previously restricted by
robustness of local SIFT-matching and to improve the model fitting accuracy when compared
to state-of-the-art multi-model fitting techniques.
\end{abstract}

\section{Introduction}

Many existing methods for model fitting and 3D structure estimation use
pre-matched image features as an input (bundle adjustment  \cite{triggs:00}, homography fitting \cite{lowe:ijcv07,suter:iccv09}, rigid motion estimation \cite{Tomasi:92,Costeira:95,Vidal:08}).
Vice versa, many matching methods (sparse/dense stereo) often use some pre-estimated structural constraints, e.g.~epipolar geometry, to identify correct matches/inliers.
This paper introduces a novel framework for simultaneous estimation of
high-level structures (multi-model fitting) and low-level correspondences (feature matching).
We discuss several regularization-based formulations of the proposed {\em fit~\&~match} (FM) problem.
These formulations use a generalization of the {\it assignment problem} and
we use efficient specialized {\it min-cost-max-flow} solver that has been overlooked in the computer vision community.
This paper primarily focuses on jointly solving multi-homography fitting and sparse feature matching
as a simple show case for the FM paradigm. Other applications would be rigid motions estimation, camera pose estimation \cite{softposit}, etc.

\paragraph{Related Work:}
In case of reliable matching, RANSAC is a well-known robust method for single model fitting. The main idea is to generate a number of model proposals by randomly sampling the matches and then select one model with the largest set of inliers (a.k.a. consensus set) with respect to some fixed threshold.
In case of unreliable matching, e.g.~repetitive texture or wide view-point, RANSAC or any technique that relies on pre-computed matching would fail.

Guided-MLESAC \cite{tordoff2005guided} and PROSAC \cite{chum2005matching} are RANSAC generalizations that try to overcome unreliable matches while generating model hypotheses.
Their main idea is to ensure that matches with high matching scores are more likely to get sampled,
thus ``guiding'' the sampling process while generating model hypotheses.
One could argue that these techniques would still fail since false matches could also have high matching scores, e.g.~scenes with repetitive texture.
SCRAMSAC \cite{sattler2009scramsac} is a form of spatial guided sampling that uses a spatial consistency filter to restrict the sampling domain to matches with similar local geometric consistency. This method is sensitive to the ratio of occluded/unoccluded features, as in that case the assumption that correct matches form a dense cluster is no longer valid.
The main drawback of these RANSAC generalizations is that they focus on generating a reliable model hypotheses by using pre-matched features (fixed matching). That drawback could be avoided by jointly solving the matching and fitting problems.

An attempt to formulate an objective function for {\it fitting-\&-matching} naturally leads to a version of the assignment problem.
The majority of prior work could be divided into two major groups: matching techniques using quadratic assignment problems and FM techniques using linear assignment subproblems.

{\it Quadratic assignment problem} (QAP) normally appears in the context of non-parametric matching. For example, the methods in \cite{VK:PAMI12,berg2005shape,leordeanu2005spectral} estimate non-rigid motion correspondences as a sparse vector field.
They rely on a quadratic term in the objective function to encourage geometric regularity between identified matched pairs. Such QAP formulations often appear in shape matching and object recognition.
QAP is NP-hard and these methods use different techniques to approximate it.
For example, \cite{Rang96graduated} approximates QAP by iteratively minimizing its first-order Taylor expansion, which reduces to a {\it linear assignment problem} (LAP).

If correspondences are constrained by some parametric model(s), matching often simplifies to LAP when model parameters are fixed.
In this case, the geometric regularity is enforced by a model fidelity term (linear w.r.t.~matching variables) and pair-wise consistencies \cite{VK:PAMI12,berg2005shape,leordeanu2005spectral} are no longer needed.
Typically for FM problems, LAP-based feature matching and model parameter fitting are preformed in a {\it block coordinate descent} (BCD) fashion.
For example, SoftPOSIT \cite{softposit} matches 2D image features to 3D object points and estimate camera pose in such iterative fashion. Building on the ideas in SoftPOSIT Serradell et al.~\cite{serradell2010combining} fit a single homography using geometric and appearance priors with unknown correspondences.
%It should be noted that although SoftPOSIT solves for the matching by using {\it softassign} \cite{Rang96graduated}---which will result in an approximate matching. SoftPOSIT justified using {\it softassign} by claiming that the objective function is non-linear. That claim is not exactly accurate as SoftPOSIT uses BCD to minimize that non-linear objective function. Thus, when solving for the matching all the other parameters are fixed and the objective function becomes linear in terms of the matching variables.
SoftPOSIT utilizes is a smoothing technique that tries to move from one suboptimal solution of a smoothed version of the objective to another less smoothed one by decreasing the temperature, i.e. smoothing factor. Their technique does not guarantee global optimal, can not handle multiple models, and it is sensitive to the temperature update factor.

Our work develops a generalization of {\it linear assignment problem} for solving FM problem when matching is constrained by an unknown number of geometric models.
In contrast to \cite{softposit,serradell2010combining}, we do not assume that matches/correspondences are constrained by a single parametric model.
Note that in order to solve FM problem for multi-models, a regularization term is required to avoid over fitting. Unlike \cite{serradell2010combining,softposit,VK:PAMI12}, our energy formulation includes label cost regularization as in \cite{LabelCosts:IJCV12}.

Another related approach, {\it guided matching}, is a post-processing heuristic for increasing the number of matches in case of single model fitting \cite{HZ:03}.
Similar to our approach, guided matching iteratively re-estimates matches and refines the model.
In contrast to our approach, guided matching pursues different objectives at refitting and re-matching steps\footnote{Geometric errors minimization vs. inliers maximization.} and does not guarantee convergence.
Our method could be seen as an energy-based guided matching with guaranteed convergence.
Moreover, unlike guided matching \cite{HZ:03}, our regularization approach is designed for significantly harder problems where data supports multiple models.

\paragraph{Contribution:}
In this work we propose two FM energy functionals~\eqref{eq:DL} and~\eqref{eq:DVL} for jointly solving matching and multi-model fitting.
Energy~\eqref{eq:DL} consists of two terms: unary potentials for matching similar features and assigning matched features to their best fitting geometric models,
and a label cost term to discourage overfitting by penalizing the number of labels/models assigned to matches.
Energy~\eqref{eq:DVL} consists of unary potential and label cost terms, as in energy~\eqref{eq:DL}, and a pairwise potential term for encouraging nearby matches to be assigned to the same label/model.

The key sub-problem when minimizing~\eqref{eq:DL} or~\eqref{eq:DVL} in BCD fashion is to solve multiple {\it generalized assignment problem}  (GAP), which is our novel generalization of LAP to multi-model case, problems efficiently.
Regularized GAP jointly formulates feature-to-feature matching and match-to-model assignment while penalizing the number of models assigned to matches.
We propose a fast approach to solve multiple similar GAP instances efficiently, by using {\it min-cost-max-flow} and flow recycling.

Figure~\ref{fig:metronCollgeZoomed} compares the results of a standard energy-based multi-model fitting algorithm \cite{LabelCosts:IJCV12} (EF) and our proposed energy-based multi-model {\it fitting-\&-matching} algorithm (EFM).
EF used the standard pre-matching technique in \cite{sift:04} that rejected a relatively large number of true matches.
EFM found better models' estimates because it nearly doubled the number of identified matches.

\section{Our Approach}
Standard techniques for sparse feature matching \cite{sift:04} independently decide each match relying on the discriminative power of the used feature descriptor. These techniques are prone to ignoring a large number of non-distinct image features that could have been valid matches. Our unified framework simultaneously estimate high-level structures (multi-model fitting) and low-level correspondences (features matching). Unlike standard techniques, our approach is less vulnerable to the descriptor's discriminative power. We discuss regularization-based formulation of the proposed {\it fit \& match} problem. While there are many different applications for a general FM paradigm, this work primarily focuses on jointly solving geometric multi-model fitting (homographies) and sparse feature matching.

We will use the following notations in defining our energy:
\begin{center}
\begin{tabular}{ccl}
$\lF$ & - & set of all observed features in the left image. \\
$\rF$ & - & set of all observed features in the right image. \\
$\labelset$   & - & a set of randomly sampled homographies (labels).\\
$\labelvar_p$ & - & label assigned to feature $p$ such that $\labelvar_p \in \labelset$\\
$\labelvars$  & - & a {\em labelling} of all features in the left image, $\labelvars=\{\labelvar_p| p \in \lF\}$\\
$\param_h$ & - & parameters of homography $h$ from left image to right image. \\
$\param$ & - & set of all models' (homographies) parameters. \\
$\subL$ & - & Subset of features in the left image supporting one geometric model\\
     &  & (plane, homography), see Figs (\ref{fig:correspondingSets})(a-b).\\
$\subR$ & - & Subset of features in the right image supporting one geometric model\\
 &  & (plane, homography), see Figs (\ref{fig:correspondingSets})(a-b).\\
$x_{pq}$ & - &is a binary variable which is 1 if $p$ and $q$ are matched (assigned)  \\
     &  & to each other and 0 otherwise.\\
$\M$ &:= & $\{x_{pq}\;|\;(p,q) \in \lF \times \rF\}$.\\
$Q(p,q)$ & - & appearance penalty for features $p\in \lF$ and $q\in \rF$ \\ & & based on similarity of their descriptors. \\
$\N$& - & edges of near-neighbour graph, e.g. Delaunay triangulation, \\
     &  &       for left image features.\\
\end{tabular}
\end{center}

\subsection{Energy}
We will define the overall matching score between two features $p\in \lF$ and $q\in \rF$ as a function of geometric transformation $\param_h$
%\vspace{-0.05in}
\begin{equation}
D_{pq}(\param_h)=||\param_h\cdot p-q|| + Q(p,q) \label{eq:nonsymmatchingscore}
\end{equation}
\vspace{-0.05in}
combining the geometric error and the appearance penalty where $||\;||$ denotes geometric transfer error. A similar matching score was used in computing the ground truth matching in \cite{Mikolajczyk:CVPR2003,YanKe:CVPR2004}.
We can also use a symmetric matching score
%\vspace{-0.1in}
\begin{equation}
D_{pq}(\param_h)=||\param_h\cdot p-q|| + ||\param_h^{-1}\cdot q-p||+ Q(p,q).
\label{eq:symmatchingscore}
\end{equation}
%\vspace{-0.05in}
We are only interested in symmetric appearance penalty $Q(p,q)$, e.g. the angle (or some metric distance) between the features' descriptors of $p$ and $q$. From here on $D_{pq}$ refers to the symmetric matching score.

In this work, $Q(p,q)=0$ if the angle between the two features' descriptors is less than $\pi/4$ and $\infty$ otherwise. The aforementioned non-continuous appearance penalty is less sensitive to the descriptor's discriminative power in comparison to the continuous one.

To simplify our formulation we will introduce our energies under the assumption that there are no occlusions
\begin{eqnarray}
E_1(\labelvars,\params,\M) &= &\sum_{\substack{p \in \lF\\ q \in \rF}} D_{pq}(\param_{\labelvar_p})\cdot x_{pq} + \beta \sum_{h\in\labelset} \delta_h(\labelvars) \label{eq:DL}\\
\textrm{s.t.} & &\left.\begin{array}{ll}
                \sum_{p \in \lF} x_{pq} =1  &\quad \forall q \in \rF  \\
                \sum_{q \in \rF} x_{pq} =1  &\quad \forall p \in \lF \\
                x_{pq} \in \{0,1\} &\quad \forall p \in \lF,\;\forall q \in \rF
                \end{array}\right\}\label{eq:121Constrains_M}
\end{eqnarray}
where $\delta_h(\labelvars) = [\exists p\in \lF: \labelvar_p=h]$ and $[\cdot]$ are {\em Iverson brackets}, and
\begin{eqnarray}
 E_2(\labelvars,\params,\M) &= &\sum_{\substack{p \in \lF\\ q \in \rF}} D_{pq}(\param_{\labelvar_p})\cdot x_{pq}+\lambda \sum_{\pq\in\N} [\labelvar_p\neq \labelvar_q] + \beta \sum_{h\in\labelset} \delta_h(\labelvars) \label{eq:DVL}
\end{eqnarray}
under constraints~\eqref{eq:121Constrains_M}. We will show how to handle outliers/occlusions later on. $E_2$ is more powerful than $E_1$ because the spatial regularizer eliminates the artifacts that results from using only one regularizer in $E_1$. The reader is refereed to \cite{LabelCosts:IJCV12} for a more detailed discussion comparing $E_1$ and $E_2$ for fixed matching in the context of multi-model fitting.

\subsection{Optimization}
In this section, we describe an efficient approach, EFM$_1$, to minimize $E_1$ in a {\it block coordinate decent} (BSD) fashion, and a second approach, EFM$_2$, to minimize $E_2$.% that uses the result of  EFM$_1$ as an initial solution.

\vspace{0.2in}
\begin{algorithm}[H]
{\small
\DontPrintSemicolon
\textbf{Initialization:} Find an initial $\M$ using standard matching techniques\;
\textbf{repeat}\;
\hspace{0.2in} Given $\M$, solve~\eqref{eq:lc_givenM} using {\PEARL} \cite{LabelCosts:IJCV12} for $\labelvars$ and $\params$\;
\hspace{0.2in} Given $\params$, solve~\eqref{eq:lc_givenParams}-\eqref{eq:121Constrains_M_F} using LS-GAP, see Sec.~\ref{sc:ls-gap}, for $\M$ and $\labelvars$\;
\textbf{until} $E_1$ converges\;
}\NoCaptionOfAlgo
\caption*{\bf{E}nergy-based {\it \bf{F}itting \& \bf{M}atching for $E_1$ (EFM$_1$)}\label{alg:E1}}
\end{algorithm}
\vspace{0.2in}

EFM$_1$ finds an initial matching using standard matching techniques and then it iteratively minimizes $E_1$  by alternating between solving for $\labelvars$ and $\params$ while fixing $\M$, and solving for $\labelvars$ and $\M$ while fixing $\params$. Although EFM$_1$ is guaranteed to converge since $E_1$ is bounded below, i.e.~$E_1 \geq \beta$, it is not trivial to derive a theoretical bound on the convergence rate and approximation ratio of EFM$_1$. However, in Section~\ref{sc:evaluation}, we empirically show that EFM$_1$ converges in a few iterations to a near optimal solution.

On the one hand, $E_1$ for fixed $\M$ reduces to
\begin{equation}
E(f,\params) = \underset{p \in \lF}{\sum}D_{p}(\param_{\labelvar_p})+ \beta \sum_{h \in\labelset} \delta_h(\labelvars) \label{eq:lc_givenM}
\end{equation}
where  $D_p(\param_h) = D_{pq}(\param_h) \;\; \forall h \in \labelset$  provided that $q$ is assigned to $p$ by $\M$, i.e.~ $x_{pq}=1$.
Furthermore, energy~\eqref{eq:lc_givenM} could be efficiently solved for $\labelvars$, $\params$ using {\PEARL} \cite{LabelCosts:IJCV12}.

%On the other hand, $E_1$ for fixed $\params$ could be re-parameterized and written in the following form
%\begin{eqnarray}
%E(\M_\labelvars) = &\underset{h\in\labelset}{\sum}\; \underset{p \in \lF}{\sum} \; \underset{q \in \rF}{\sum}D_{pq}(\param_h) x_{pqh} +  \beta \sum_{h \in\labelset} \delta_h(\M_\labelvars) \label{eq:lc_givenParams} \\
%\textrm{s.t.}\quad   & \left.\begin{array}{ll}
%                         \underset{h\in\labelset}{\sum}\;\underset{ p \in \lF}{\sum} x_{pqh} =1  & \forall q \in \rF \\
%                         \underset{h\in\labelset}{\sum}\;\underset{ q \in \rF}{\sum} x_{pqh} =1  & \forall p \in \lF \\
%                         x_{pqh} \in \{0,1\} & \forall h \in \labelset,\; p \in \lF,\; q \in \rF
%                        \end{array}\right\}\label{eq:121Constrains_MF}
%\end{eqnarray}
%where binary variable $x_{pqh}$ is 1 if $p$ and $q$ are matched to each other and assigned to model $h$, and 0 otherwise. Matching $\M_\labelvars$ is defined as $\{x_{pqh}\;|\;(p,q,h) \in \lF\times \rF \times \labelset\}$ encapsulating information of both feature-to-feature and match-to-model assignments, and $\delta_h$ is now defined as $\delta_h(\M_\labelvars) = [\exists p\in \lF,\; q\in \rF: x_{pqh}=1]$.
On the other hand, $E_1$ for fixed $\params$ reduces to
\begin{eqnarray}
E(\labelvars,\M) = & \underset{p \in \lF}{\sum} \; \underset{q \in \rF}{\sum}D_{pq}(\param_{f_p}) x_{pq} +  \beta \sum_{h \in\labelset} \delta_h(\labelvars) \label{eq:lc_givenParams} \\
\textrm{s.t.}\quad   & \left.\begin{array}{ll}
                         \underset{ p \in \lF}{\sum} x_{pq} =1  & \forall q \in \rF \\
                         \underset{ q \in \rF}{\sum} x_{pq} =1  & \forall p \in \lF \\
                         x_{pq} \in \{0,1\} & \forall \; p \in \lF,\; q \in \rF.
                        \end{array}\right\}\label{eq:121Constrains_M_F}
\end{eqnarray}
We will refer to the special unregularized case of optimization problem~\eqref{eq:lc_givenParams}-\eqref{eq:121Constrains_M_F} where $\beta=0$ as the {\it generalized assignment problem} (GAP) \footnote{Our definition of GAP is different from the formal definition in optimization literature.}.

This is a weighted matching problem over a fixed set of multiple models that match features and assigns each match to a model.
GAP is an {\it integral linear program}, see Appendix~\ref{ap:TMproof} for proof, and therefore any Linear Programming toolbox could be used to find its optimal solution by solving its relaxed LP---but will be considerably slow due to the size of the problem at hand.
A fast approach to solve GAP is described in Section~\ref{sc:gap_via_reduction}.
The optimal solution for GAP may overfit models to data since the number of models is not regularized when $\beta=0$. For $\beta>0$ optimization problem~\eqref{eq:lc_givenParams}-\eqref{eq:121Constrains_M_F} could be solved using LS-GAP, introduced in Section~\ref{sc:ls-gap}, which utilizes a GAP solver in a combinatorial local search algorithm.
This local search over different subsets of $\labelset$ selects a solution reducing energy~\eqref{eq:lc_givenParams}.

It should be noted that EFM$_1$ requires initial matching. To overcome this draw back it is possible to just randomly sample models from the set of all possible matches. Then alternate between fixing $\theta$ to solve~\eqref{eq:lc_givenParams}-\eqref{eq:121Constrains_M_F} for $\M$ and $\labelvars$ using H-GAP, introduced in Section~\ref{sc:H-GAP}, and fixing $\M$ and $\labelvars$ to solve for $\theta$ using Levenberg-Marquardt. H-GAP idea is based ons a similar a greedy heuristic method \cite{HeuristicUFLP,cornuejols1983uncapacitated,Hochbaum82} that finds an approximate solution for the Uncapacitated Facility Location Problem. One major difference between LS-GAP and H-GAP is that the latter has very small upper bound on the number of iterations for termination compared to LS-GAP.

\vspace{0.2in}
\begin{algorithm}[H]
{\small
\DontPrintSemicolon
\textbf{Initialization:} Find an initial $\M,\labelvars,\params$ using EFM$_1$\;
\textbf{repeat}\;
\hspace{0.2in} Given $\M$, solve \eqref{eq:smth_lc_givenM} using {\PEARL} for $\labelvars$ and $\params$\;
\hspace{0.2in} Given $\labelvars,\params$, solve \eqref{eq:smth_lc_givenParamsLabelvars}-\eqref{eq:121Constrains_M} using LC-GAP, see Sec.~\ref{sc:lc-gap}, for $\M$\;
\textbf{until} $E_2$ converges \;
}\NoCaptionOfAlgo
\caption*{\bf{E}nergy-based {\it \bf{F}itting \& \bf{M}atching for $E_2$ (EFM$_2$)}\label{alg:E2}}
\end{algorithm}
\vspace{0.2in}

EFM$_2$ uses EFM$_1$ result as an initial solution and then iteratively minimizes $E_2$  by alternating steps solving for $\labelvars$ and $\params$ while fixing $\M$, and solving for $\M$ while fixing $\params$ and $\labelvars$. Energy~\eqref{eq:DVL} for fixed $\M$  reduces to
\begin{equation}
E(f,\params) = \underset{p \in \lF}{\sum}D_{p}(\param_{\labelvar_p})+ \lambda \sum_{\pq\in\N} [\labelvar_p\neq \labelvar_q] +
\beta \sum_{h \in\labelset} \delta_h(\labelvars) \label{eq:smth_lc_givenM}
\end{equation} and is solved using \PEARL.
For fixed $\labelvars$ and $\params$ energy~\eqref{eq:DVL} reduces to
\begin{eqnarray}
E(\M) =\underset{p \in \lF}{\sum} \; \underset{q \in \rF}{\sum}D_{pq}(\param_{\labelvar_p}) x_{pq}\label{eq:smth_lc_givenParamsLabelvars}
\end{eqnarray}
under constraints~\eqref{eq:121Constrains_M}. Energy~\eqref{eq:smth_lc_givenParamsLabelvars} is solved using label constrained GAP (LC-GAP), Section~\ref{sc:lc-gap}. LC-GAP is a variant of the fast GAP solver that can change feature-to-feature matching without affecting the current labelling. It should be noted that, based on our experience, EFM$_2$ slightly modifies the initial solution by rejecting or correctly matching less than a handful of false positives. Therefore, in practice we suggest to run only one iteration of EFM$_2$ to reject the false positives incorrectly matched due to lack of spatial coherency.

Due to occlusions $|\lF|\!\!\neq\!\!|\rF|$ and that renders \eqref{eq:DL}-\eqref{eq:121Constrains_M} unfeasible since the one-to-one constraints could never be met.
We add $||\lF|\!-\!|\rF||$ dummy features, with a fixed matching cost $T$, to the smaller set of features to ensure feasibility.
This is equivalent to changing a rectangular assignment problem to a square one.
%Without loss of generality, we could assume that $|\lF|=|\rF|$ if not we add dummy features to the smaller set with a fixed matching cost $T$.
%This is equivalent to changing a rectangular assignment problem to a square one.
Also, to make our approach robust to outliers we introduce an outlier model $\phi$ such that $D_{pq}(\phi)\!=\!T$ for any $p \!\in\! \lF$ and $q \!\in\! \rF$.
The use of an outlier model with a uniformly distributed cost $T$ is a common technique in Computer Vision \cite{LabelCosts:IJCV12,PEARL:IJCV12}.

\newpage
\section{Algorithms}
\label{sc:alg}
%We introduce an optimal GAP solver, in Section~\ref{sc:gap}, for energy~\eqref{eq:lc_givenParams}-\eqref{eq:121Constrains_MF} for $\beta=0$. Section~\ref{sc:ls-gap} covers {\it Local Search}-GAP (LS-GAP) algorithm which is used in EFM to find an approximate solution for energy~\eqref{eq:lc_givenParams}-\eqref{eq:121Constrains_MF} for $\beta>0$ by solving a series of similar GAP instances. It should be noted that GAP could be reduced to LAP since the $\labelvar$ and $\M$ varibales are independent. For any given possible match $x_{pq}$ it could only be assigned to the model with the lowest matching cost. In other words, unregularized GAP could be reduced to a regular assingment problem by selecting the model with lowest cost for every possible match. Therefore, any LAP solver could be used in LS-GAP by solving the GAP instances indepedntly.
%However, our {\it min-cost-max-flow} GAP solver is a better alternative as it could be used to solve a series of similar GAP instances more effeciently by recycling the flow from one instance to another.
In Section~\ref{sc:gap}, we will give a brief overview of the {\it min-cost-max-flow} (MCMF) problem, and its {\it Successive Shortest Path} (SSP) algorithm \cite{networkflows93} in order to introduce our {\em flow recycling} technique for efficiently solving similar GAP instances. Then we will describe two ways to solve GAP using a MCMF solver. Section~\ref{sc:ls-gap} covers {\it Local Search}-GAP (LS-GAP) algorithm which is used in EFM to find an approximate solution for energy~\eqref{eq:lc_givenParams}-\eqref{eq:121Constrains_M_F} for $\beta>0$ by solving a series of similar GAP instances. As an alternative to LS-GAP, Section~\ref{sc:H-GAP} covers a greedy heuristic algorithm H-GAP. Section~\ref{sc:lc-gap} covers LC-GAP which is a variant of LS-GAP that can change feature-to-feature matching without affecting the current labelling.

\subsection{Solving GAP}
\label{sc:gap}
We will describe our most recent method for solving GAP in Section~\ref{sc:gap_via_reduction} and then in Section~\ref{sc:gap_no_reduction} we describe an earlier method that we previously used to solve GAP, for the sake of completeness.
It should be noted that our {\em flow recycling} for solving a series of similar GAP instances could used with any MCMF solver \cite{networkflows93} or even weighted bipartite matching algorithms \cite{AP:Burkard}, it is not restricted to SSP. For simplicity, we discuss flow recycling in the context of SSP.

\subsubsection{Min-Cost-Max-Flow Problem (overview)}
MCMF problem is defined as follows. Let $\graph=(\nodes,\edges)$ denote a graph with  vertices $\nodes$
and edges $\edges$ where each edge $(v,w)\in \edges$ has a capacity $u(v,w)$ and cost $c(v,w)$.
Let $\flow$ be a flow function such that $0 \leq \flow(v,w) \leq u(v,w)$ for over all edges in $\edges$. The cost of an arbitrary flow function $\flow$ is defined as $cost(\flow)=\sum_{(v,w)\in \edges} c(v,w)\cdot\flow(v,w)$. MCMF is a valid maximum flow $\flow$ from $s$ to $t$ in $\nodes$ that has minimum cost.

We will limit our discussion on SSP \cite{networkflows93} to a certain type of graphs; unit capacity edges, and complete bipartite graphs with $|\nodes|=2n$ and $|\edges|=n^2$. SSP for solving MCMF on general graphs is beyond the scope of our discussion. SSP successively finds the shortest path w.r.t.~edge costs from source to sink and augments these paths until the network is saturated. For unit capacity graphs, augmentation of an edge reverses its direction and flips its cost sign. Finding the shortest path with negative costs is expensive. Instead of the original costs SSP uses {\it reduced costs} $$c^\pi(v,w):=c(v,w)-\pi(v)+\pi(w) \ge 0$$ where $\pi(v)$ is the {\it potential} of node $v$. Initially set to zero, node potentials are updated after each path augmentation to ensure that the reduced costs non-negativity constraints are satisfied: $\pi(v)=\pi(v)-d(v)$ for all $v$ in $\nodes$ where $d(v)$ is the shortest distance cost w.r.t.~$c^\pi$ from the $s$ to $v$. A shortest path w.r.t.~$c^\pi$ could be found in $O(n^2)$ using Dijkstra's algorithm. As we are dealing with complete bipartite graphs, we need to find $n$ paths to saturate the the network between $s$ and $t$. Thus, SSP is $O(n^3)$ for unit capacity complete bipartite graphs with $|\nodes|=2n$ and $|\edges|=n^2$.

\subsubsection{Solving GAP via Reduction to LAP}
\label{sc:gap_via_reduction}
GAP~\eqref{eq:lc_givenParams}-\eqref{eq:121Constrains_M_F} reduces to LAP since $\labelvar$ and $\M$ are independent:
any pair $(p,q)$ has optimal label $\labelvar_p= \underset{{h\in\labelset}}{arg \min}\;D_{pq}(\theta_h)$ independently from the value of $x_{pq}$.
A simple proof by contradiction shows that the previous statement is true. Given an optimal GAP solution where $x_{pq}$ was assigned to label $k$ such that $D_{pq}(\theta_k) > \underset{{h\in\labelset}}{\min}\;D_{pq}(\theta_h)$ then the GAP solution is not optimal as we could decrease the energy by assigning $x_{pq}$ to model $k^*=\underset{{h\in\labelset}}{arg \min}\;D_{pq}(\theta_h)$ without violating any of the linear constraints~\eqref{eq:121Constrains_M_F}.
The optimal $\M$ in~\eqref{eq:lc_givenParams}-\eqref{eq:121Constrains_M_F} is found by solving the following LAP
\begin{equation}
E(\M)=\sum_{p \in \lF,\; q \in \rF} D_{pq}\cdot x_{pq}\label{eq:rgap}
\end{equation}
subject to~\eqref{eq:121Constrains_M_F} where $D_{pq}:=\underset{{h\in\labelset}}{\min}\;D_{pq}(\theta_h)$.
In other words, unregularized GAP could be reduced to a regular assignment problem by selecting the model with lowest cost for every possible match.

LAP~\eqref{eq:rgap}-\eqref{eq:121Constrains_M_F} can be equivalently formulated as a standard {\em min-cost-max-flow} (MCMF) problem with known efficient
solvers \cite{AP:Burkard,networkflows93}. To formulate LAP~\eqref{eq:rgap}-\eqref{eq:121Constrains_M_F} as MCMF problem we build graph $\graph\!\!=\!\!(\nodes,\edges)$ with nodes
\begin{align*}
    \nodes=&\{ s,t\} \cup \{p \;|\; p\in \lF\} \cup \{q \;|\;q\in \rF\},
\vspace{-1.5ex}
\end{align*}
edges
$$\begin{matrix*}[l]
    \edges=\!\!\!\!\!&\{(s,p),(q,t),(p,q) | p\in \lF, q\in \rF\},\\
\end{matrix*}$$
capacity $u(v,w)\!=\!1$ for all edges $ (v,w) \in \edges$, and cost $c(p,q) =D_{pq}$ for edges $(p,q) \in \lF\times \rF$ and 0 for other edges.
The optimal $\M$ and $\labelvars$ for GAP can be obtained from MCMF flow $\flow^*$ for $\graph$ as $x_{pq}=\flow^*(p,q)$ for all $(p,q) \in \lF\times \rF$ and $f_p\!=\!\underset{{h\in\labelset}}{arg \min}\;D_{pq}(\theta_h)$ if $p,q$ are matched, $x_{pq}\!=\!1$.

{\small \paragraph{Solving a Series of GAPs:}}
We propose $O(n^2)$ method for solving MCMF corresponding to a modified LAP~\eqref{eq:rgap}-\eqref{eq:121Constrains_M_F}
after changing one or all edge costs associated with one feature in $\lF$.
Assume MCMF $\flow$ for $\graph$ and node potential function $\pi$ that satisfy the reduced costs non-negativity constraints on the residual
graph $\graph_\flow$.
Changing edge costs associated with feature $p$ may violate reduced cost non-negativity constraints involving $p$.
To regain feasibility after dropping the no longer needed artificial nodes $s$ and $t$ and their edges, we reverse the flow through $(p,q)$ where $p$ and $q$ are matched by $\flow$ and update $\pi(p)$
$$\pi(p)= \min c(p,v)+\pi(v)\;\;\;\; \forall v \in \rF.$$ Finally, we push one unit of flow from $p$ to $q$, i.e.~find the shortest path w.r.t.~$c^\pi$, to maximize the flow. The reduced cost optimally theorem \cite{networkflows93} grantees that the resulting flow is MCMF.
In case $m$ features in $\lF$ had their associated costs changed, the new MCMF could be found in $O(mn^2)$ by
applying the steps above sequentially to each feature.
These steps could be used with any LAP \cite{AP:Burkard} or MCMF solver not just SSP. Given an optimal solution for LAP~\eqref{eq:rgap}-\eqref{eq:121Constrains_M_F}, it
is possible to compute the optimal node potentials that satisfy reduced cost non-negativity constraints in polynomial time \cite{networkflows93}. Given a MCMF $\flow$, first we build the residual graph $\graph_\flow$ which does not contain any negative cycles otherwise $\flow$ is not MCMF. Then we compute the shortest distance $d$, w.r.t.~the edge costs $c$, between a node in $\graph_\flow$ and all the other nodes using Bellman and Ford. Notice that the range of the edge $c$ function of $\graph_\flow$  is not guaranteed to be non-negative. However, $\graph_\flow$ contains no negative cost cycles otherwise of $\flow$ is not MCMF\footnote{Notice that if there exists a negative cost cycle we could augment the flow along that cycle and reduce the flow cost.}. The distance $d$ is well defined since there are no negative cost cycles in $\graph_\flow$ thus $d(w) \leq d(v) + c (v,w)$  for all edges $(v,w)$ in $\graph_\flow$. We could defined the nodes' potentials as $\pi=-d$ thus
\begin{align*}-\pi(w) & \leq -\pi(v) + c(v,w)\\
                    0 &\leq c(v,w)-\pi(v) +\pi(w).
\end{align*}

On a final note, a sparse weighted bipartite matching solver is the best method to solve LAP~\eqref{eq:rgap}-\eqref{eq:121Constrains_M_F} after removing edges with cost $T$, i.e.~outliers/occlusions. However, the flow recycling should be modified accordingly to cope with an incomplete bipartite graph. In SSP, the node potentials are updated using a well defined the shortest distance function $d$,~i.e. $d(w) \leq d(v) + c (v,w)$  for all edges $(v,w)$ in $\graph$. If two adjacent nodes are not reachable from the node that we are computing the distance to, i.e.~their distances are infinity, then distance function is no longer well defined.

\subsubsection{Solving GAP Directly}
\label{sc:gap_no_reduction}
Energy~\eqref{eq:lc_givenParams}-\eqref{eq:121Constrains_M_F} could be re-parametrized and written in the following form
\begin{eqnarray}
E(\M_\labelvars) = &\underset{h\in\labelset}{\sum}\; \underset{p \in \lF}{\sum} \; \underset{q \in \rF}{\sum}D_{pq}(\param_h) x_{pqh} +  \beta \sum_{h \in\labelset} \delta_h(\M_\labelvars) \label{eq:lc_givenParams_MF} \\
\textrm{s.t.}\quad   & \left.\begin{array}{ll}
                         \underset{h\in\labelset}{\sum}\;\underset{ p \in \lF}{\sum} x_{pqh} =1  & \forall q \in \rF \\
                         \underset{h\in\labelset}{\sum}\;\underset{ q \in \rF}{\sum} x_{pqh} =1  & \forall p \in \lF \\
                         x_{pqh} \in \{0,1\} & \forall h \in \labelset,\; p \in \lF,\; q \in \rF
                        \end{array}\right\}\label{eq:121Constrains_MF}
\end{eqnarray}
where binary variable $x_{pqh}$ is 1 if $p$ and $q$ are matched to each other and assigned to model $h$, and 0 otherwise. Matching $\M_\labelvars$ is defined as $\{x_{pqh}\;|\;(p,q,h) \in \lF\times \rF \times \labelset\}$ encapsulating information of both feature-to-feature and match-to-model assignments, and $\delta_h$ is now defined as $\delta_h(\M_\labelvars) = [\exists p\in \lF,\; q\in \rF: x_{pqh}=1]$.
To formulate GAP as MCMF problem we build graph $\graph^*=(\nodes,\edges)$, see Fig.~\ref{fig:MCMF_Config}(a), with the set of nodes
\begin{align*}
    \nodes=&\{ s,t\} \cup \{n_p \;|\; p\in \lF\} \cup \{n_q \;|\;q\in \rF\} \; \cup\\
          &\{n_{ph} \;|\; p\in \lF, h\in \labelset \} \cup \{n_{qh} \;| \;q\in \rF, h\in \labelset\},
\end{align*}
the set of edges
$$\begin{matrix*}[l]
    \edges=&\{(s,n_p) &|\; p\in \lF\} \; \cup\\
           &\{(n_p,n_{ph}) &|\; p\in \lF,h\in \labelset\} \; \cup\\
           &\{(n_{ph},n_{qh}) &|\;p\in \lF,q\in \rF, h\in \labelset \} \;\cup\\
           &\{(n_{qh},n_q) &|\;q\in \rF,h\in \labelset\} \; \cup\\
           &\{(n_q,t) &|\; q\in \rF\},
\end{matrix*}$$
and the following edge capacity $u$ and edge cost $c$ functions
\begin{align*}
u(v,w) =&1 \quad \quad \quad \quad \;\;\; \text{for}\;\;(v,w) \in \edges \\
c(v,w) =&
\begin{dcases*}
D_{pq}(\param_h) & for $(v,w) \in \{ (n_{ph},n_{qh})\; |\; p\in \lF, q\in \rF,  h\in \labelset\}$\\
0 & otherwise.
\end{dcases*}
\end{align*}
\begin{lemma}
\label{lm:mcmf}
The optimal solution for a feasible GAP (eq.~\eqref{eq:lc_givenParams_MF}-\eqref{eq:121Constrains_MF} for $\beta\!=\!0$)
is
$$\M_{\labelvar}=\{x_{pqh}=\mathbf{F^*}(n_{ph},n_{qh}) \;|\;  p\in \lF,q\in \rF, h\in \labelset\}$$
where $\mathbf{F^*}:\edges \rightarrow \{\text{0},\text{1}\}$ is the MCMF over graph $\graph^*$.
\end{lemma}
Using Lemma~\ref{lm:mcmf}, see proof Appendix~\ref{ap:mcmfproof}, a GAP solution $\M_\labelvars$ could be found by using an efficient MCMF algorithm~\cite{goldberg92MCMF}.
\begin{figure}[H]
\begin{center}
\begin{tabular}{cc}
\includegraphics[width=0.45\textwidth]{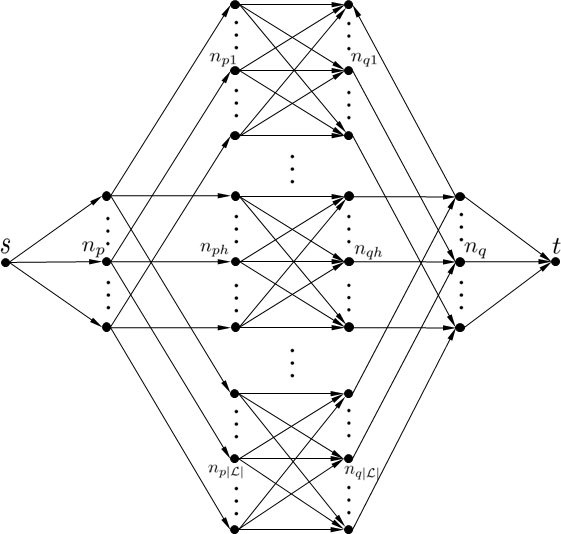}&
\includegraphics[width=0.45\textwidth]{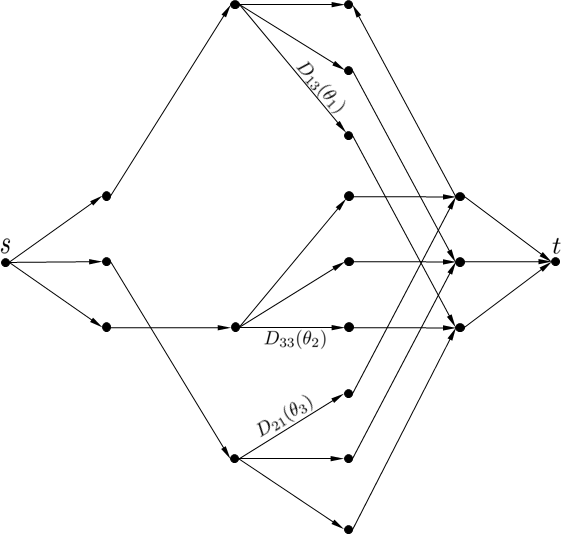}\\
{\small (a) $\graph^*$ of a generic GAP instance} &
{\small (b) Example of LC-MCMF $\graph^*_\labelvars$}
\end{tabular}
\end{center}
\vspace{-0.2in}
\caption{{\small Figure~(a) shows the generalized graph construction $\graph^*$ of a generic GAP instance---with unit capacity edges and edge cost function $c(v,w) = D_{pq}(\param_h)$ for all $(v,w) \in \{ (n_{ph},n_{qh})\;|\;\forall p\in \lF, q\in \rF,  h\in \labelset\}$ and $0$ otherwise. This construction does not assume that $|\lF|=|\rF|$. Figure~(b) shows $\graph^*_\labelvars$ of a GAP with $|\lF|=|\rF|=3$ and $|\labelset|=3$ under the labelling constraint $\labelvars=[1\; 3 \; 2]$.}}
\label{fig:MCMF_Config}
\end{figure}
\subsection{Local Search-GAP (LS-GAP)}
\label{sc:ls-gap}
Now we introduce a local search algorithm that solves regularized GAP \eqref{eq:lc_givenParams}-\eqref{eq:121Constrains_M_F} with $\beta>0$ using GAP algorithm in Section~\ref{sc:gap_via_reduction} as a sub-procedure. Assume that $\labelset$ is the current set of possible models\footnote{In practice, $\labelset$ is restricted to be the set of models that are assigned to at least one matched pair of features in energy~\eqref{eq:lc_givenM} solution.}.
Let $\labelset_c$ be an arbitrary subset of $\labelset$ and $\M_\labelvars(\labelset_c)$ denote the GAP solution when the label space is restricted to $\labelset_c$. Note that GAP ignores the label cost term in~\eqref{eq:lc_givenParams} but we could easily evaluate energy~\eqref{eq:lc_givenParams} for $\M_\labelvars(\labelset_c)$.
%The proposed LS-GAP algorthim greedily searches for an optimal subset $\labelset_c$ such that $\M_\labelvars(\labelset_c)$ minimizes energy~\eqref{eq:lc_givenParams}.
The proposed LS-GAP algorithm greedily searches over different subsets $\labelset_c$ for one such that $\M_\labelvars(\labelset_c)$ has the lowest value of energy~\eqref{eq:lc_givenParams}.
Our motivation to search for minima of energy~\eqref{eq:lc_givenParams}-\eqref{eq:121Constrains_M_F} only among GAP solutions comes from an obvious observation that a global minima of~\eqref{eq:lc_givenParams}-\eqref{eq:121Constrains_M_F} must also solve the GAP if the label space is restricted to a right subset of $\labelset$.

We define sets of all possible {\it add}, {\it delete} and {\it swap} combinatorial search moves as
\begin{eqnarray*}
\N^a(\labelset_c)&=&\cup_{h \in \labelset\setminus \labelset_c} \{\labelset_c \cup h\}\\
\N^d(\labelset_c)&=&\cup_{h \in\labelset_c}\{\labelset_c \setminus h\}\\
\N^s(\labelset_c)&=&\cup_{\substack{h\in\labelset_c\\\ell \in \labelset\setminus \labelset_c}}\{\labelset_c \cup \ell \setminus h\}.
\end{eqnarray*}
These are three different local neighbourhoods around $\labelset_c$. We also define a larger neighbourhood $\N^\star$
around $\labelset_c$ which is the union of the above
\begin{eqnarray*}
\N^\star(\labelset_c)  &=&\N^a(\labelset_c) \cup \N^d(\labelset_c) \cup \N^s(\labelset_c).
\end{eqnarray*}
LS-GAP uses a combination of {\it add}, {\it delete} and {\it swap} moves, as in \cite{LS:UFLP}, to greedily find a set of labels near current set $\labelset_t$ that is better w.r.t.~energy~\eqref{eq:lc_givenParams}.
%Let $GAP(\labelset_c)$ denote the subset of models assigned to at least one matched pair of features by solving GAP using only the models in $\labelset_c$ (not all the models in $\labelset$).

\vspace{0.1in}
\begin{algorithm}[H]
{\small
\DontPrintSemicolon
\hspace{0.2in} $\labelset_t \gets \phi$\;
\hspace{0.2in} $\N_{t} \gets \N^\star(\labelset_t)$\;
\textbf{while}$\;\exists \; \labelset_c \in \N_{t}$\;
\hspace{0.2in} \textbf{if} \text{energy~\eqref{eq:lc_givenParams} of $\M_\labelvars(\labelset_c)$} $<$ \text{energy~\eqref{eq:lc_givenParams} of $\M_\labelvars(\labelset_t)$}\;
\hspace{0.4in} $\labelset_t \gets \labelset_c$\;
\hspace{0.4in} $\N_{t} \gets \N^\star(\labelset_t)$\;
\hspace{0.2in} \textbf{else}\;
\hspace{0.4in} $\N_{t}\gets \N_{t}\setminus \labelset_c$\;
%\hspace{0.2in} \textbf{end}\;
%\textbf{end}\;
\textbf{return} the GAP solution $\M_\labelvars(\labelset_t)$\;
}\NoCaptionOfAlgo
\caption*{{\bf LS-GAP}}
\end{algorithm}
\vspace{0.1in}

In LS-GAP, we initially set $\labelset_t$  to $\phi$ but it could be any arbitrary subset of $\labelset$. Then $\N^\star(\labelset_t)$ is searched for a move with a GAP solution $\M_\labelvars(\labelset_c)$ of a lower energy \eqref{eq:lc_givenParams} than the current one, i.e $\M_\labelvars(\labelset_t)$. Such a move is accepted if it exists and $\labelset_t$ is updated accordingly otherwise LS-GAP terminates. LS-GAP will definitely converge since energy \eqref{eq:lc_givenParams} is lower bounded and $\labelset$ is finite.

To speedup our LS-GAP implementation, we construct a single graph containing all the models (not just the subset appearing in a particular GAP instance). Each particular GAP instance is solved by modifying the edge weights accordingly---edges weights $({n_p,n_{ph}})$ of any model $h$ not in the GAP instance are set to infinity. Having a single construction allows to reuse the flow (solution) from a previously solved GAP to solve the next GAP instance faster. Finally, this construction requires $O(|\labelset|)$ more space and it is slower to solve than the construction described in Section~\ref{sc:gap_via_reduction}.

\subsection{Heuristic-GAP (H-GAP)}
\label{sc:H-GAP}
Now we introduce another greedy algorithm that solves regularized GAP \eqref{eq:lc_givenParams}-\eqref{eq:121Constrains_M_F} with $\beta>0$ using GAP algorithm in Section~\ref{sc:gap_via_reduction} as a sub-procedure. Assume that $\labelset$ is the current set of randomly sampled models.
Let $\labelset_c$ be an arbitrary subset of $\labelset$ and $\M_\labelvars(\labelset_c)$ denote the GAP solution when the label space is restricted to $\labelset_c$. H-GAP terminates after at most $O(|\labelset|^2)$ iterations. We did not experiment with H-GAP.

\vspace{0.1in}
\begin{algorithm}[H]
{\small
\DontPrintSemicolon
\hspace{0.0in} $\labelset_t \gets \phi$\;
%\textbf{while}$\;\exists \; \ell \notin \labelset_t$ such that energy~\eqref{eq:lc_givenParams} of $\M_\labelvars(\ell \cup \labelset_t)$} $<$ energy~\eqref{eq:lc_givenParams} of $\M_\labelvars(\labelset_t)$}\;
\textbf{while}$\;\exists \; \ell \notin \labelset_t$ such that energy~\eqref{eq:lc_givenParams} of $\M_\labelvars(\ell \cup \labelset_t)$ $<$ energy~\eqref{eq:lc_givenParams} of $\M_\labelvars(\labelset_t)$\;
\hspace{0.2in} $\ell = \underset{k \in \{\labelset-\labelset_t\}}{arg\min}$ \text{energy~\eqref{eq:lc_givenParams} of $\M_\labelvars(k \cup \labelset_t)$}\;
\hspace{0.2in} $\labelset_t \gets \{\ell \cup \labelset_t\}$\;
%\hspace{0.2in} \textbf{end}\;
\textbf{end}\;
\textbf{return} the GAP solution $\M_\labelvars(\labelset_t)$\;
}\NoCaptionOfAlgo
\caption*{{\bf H-GAP}}
\end{algorithm}
\vspace{0.1in}

\subsection{Label Constrained-GAP (LC-GAP)}
\label{sc:lc-gap}
LC-GAP solves a GAP instance with fixed labelling $\labelvars$, i.e. each left feature $p$ must be assigned to a predefined model $\labelvar_p$. LC-GAP uses a slightly different graph construction than $\graph^*$ that enforces the required labelling constraints. The graph construction corresponding to a GAP instance under labelling $f$ constraints is $\graph^*_\labelvars=(\nodes,\edges)$ where
\begin{align*}
    \nodes=&\{ s,t\} \cup \{n_p |\forall p\in \lF\} \\
           & \cup \{n_q |\forall q\in \rF\} \; \cup \{n_{p\labelvar_p} |\forall p\in \lF\}\\
          & \cup \{n_{qh} |\forall q\in \rF, h\in \labelset\}
\end{align*}
and $\edges$, capacity function $u$ and cost function $c$ are as defined as in $\graph^*$ provided that both edge nodes exist in $\nodes$ of $\graph^*_\labelvars$, see example in Fig.~\ref{fig:MCMF_Config}(b).

\section{Ground truth}

The ground truth is computed by first manually identifying and segmenting regions corresponding to separate models (planes/homographies), see Fig.~\ref{fig:grndtruth_sample}.
Then, we compute an optimal matching of extracted features inside each identified pair of corresponding regions with respect to the geometric fitting error and appearance. This method is similar to the one used in \cite{YanKe:CVPR2004,Mikolajczyk:CVPR2003}. %{\it [I will describe how we are different from them later]}

Below we describe our technique for computing the ``ground truth matching'' for each model with manually identified spatial support, as illustrated in Fig.~\ref{fig:grndtruth_sample}. We compute sets of SIFT features $\subL$ and $\subR$ inside each pair of manually identified corresponding regions, see Fig.~\ref{fig:correspondingSets}.
It is possible to independently fit one homography to each pair of corresponding sets $\{\subL,S_r\}$.
For simplicity, we first assume that there are no occlusions, i.e if a feature appears in the left image then its corresponding feature appears in the right image and vise versa.
Thus, the number of left image features equals the number of right image features.
We will show how to deal with occlusions later.

The SIFT features of two corresponding sets, namely $\subL$ and $\subR$ see Fig.~\ref{fig:correspondingSets}, are matched using the criteria described in \cite{sift:04}. Then we use \RANSAC ~\cite{ransac:81} to find a homography $\param_h$ that maximizes the number of inliers between the features in two corresponding regions.
Using \RANSAC~in this case is not problematic since features in $\subL$ and $\subR$ support only one homography/model. This homography is only used as an initial guess in finding the ground truth model.

\begin{figure}[H]
\begin{center}
\begin{tabular}{cc}
\includegraphics[width=0.45\textwidth]{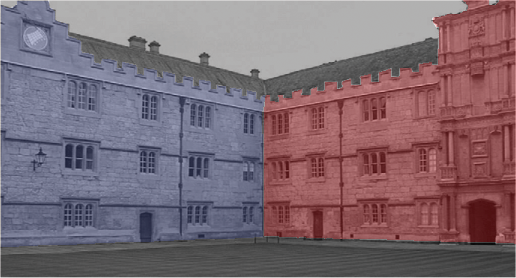}&
\includegraphics[width=0.45\textwidth]{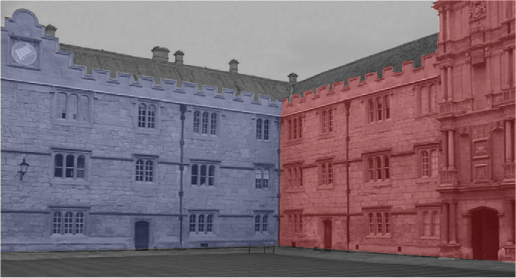}
\end{tabular}
\end{center}
\vspace{-2em}
\caption{{\small In this example we identified only two planes%\footnote{We do not compute the ground truth for the remaining planes as the number of extracted SIFT features for these planes is fairly small. Therefore the feature matching, as well as homographies, can not be estimated reliably.}
. The manually identified corresponding support regions for these two models are shown in blue and red.}}
\label{fig:grndtruth_sample}
\end{figure}

\begin{figure}[h]
\begin{center}
\begin{tabular}{cc}
\includegraphics[width=0.40\textwidth]{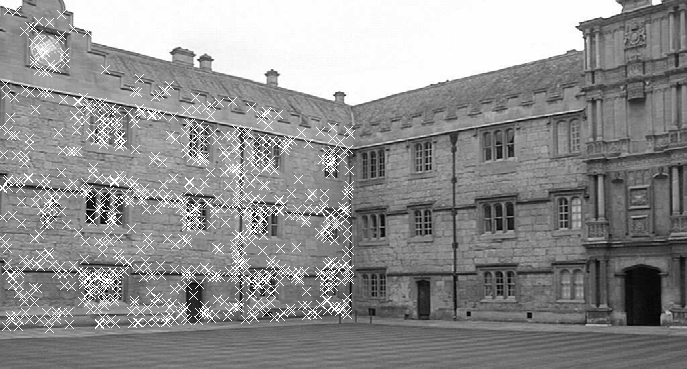}&
\includegraphics[width=0.40\textwidth]{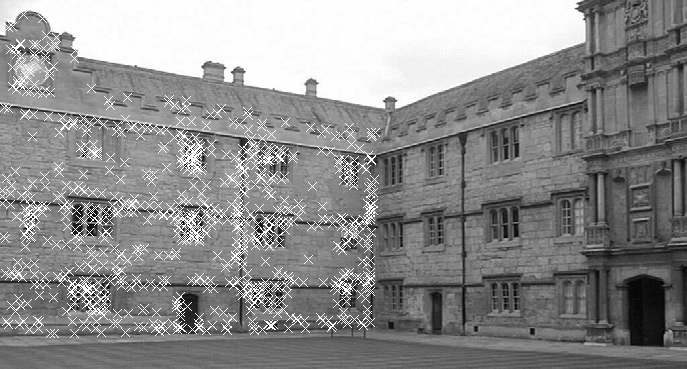}\\
\multicolumn{2}{c}{\small (a) Example of corresponding sets $\subL$ and $\subR$  supporting the blue model in Fig.~\ref{fig:grndtruth_sample}.}\\[1ex]
\includegraphics[width=0.40\textwidth]{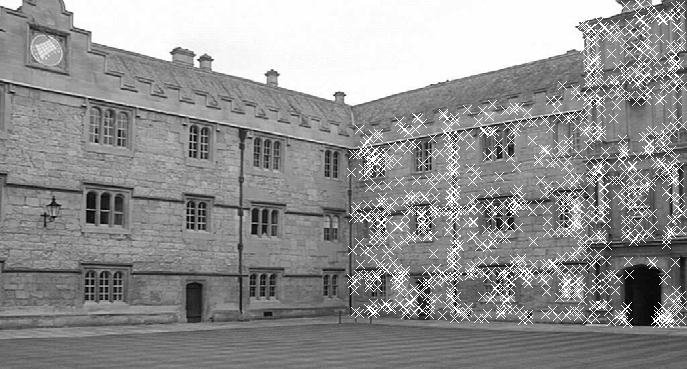}&
\includegraphics[width=0.40\textwidth]{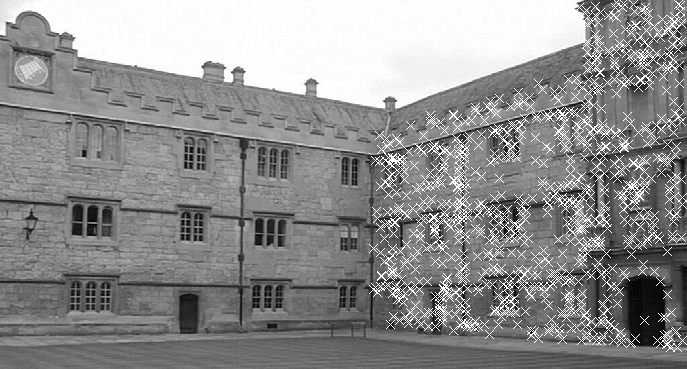}\\
\multicolumn{2}{c}{\small (b) Example of corresponding sets $\subL$ and $\subR$  supporting the red model in Fig.~\ref{fig:grndtruth_sample}.}\\[1ex]
\end{tabular}
\end{center}
\vspace{-2em}
\caption{{\small Two examples of the corresponding sets of features $\subL$ and $\subR$ supporting the blue (a) and red (b) models.}}
\label{fig:correspondingSets}
\end{figure}

Given a homography $\param_h$, the problem of finding an optimal one-to-one matching that minimizes the total sum of matching scores between the left and right features in two corresponding regions could be formulated as an assignment problem
\begin{eqnarray}
\mathbf{AP:}\quad \underset{\M}{\arg \min}\quad  &\underset{p \in \subL}{\sum} \; \underset{q \in \subR}{\sum} \unary_{pq}(\param_h)\;x_{pq}& \label{eq:gtobjctive}\\
\textrm{s.t.}\quad   & \underset{ p \in \subL}{\sum} x_{pq} =1  & \forall q \in \subR \nonumber \\
                     & \underset{ q \in \subR}{\sum} x_{pq} =1  & \forall p \in \subL \nonumber \\
                     & x_{pq} \in \{0,1\} &\forall p \in \subL,\; \forall q \in \subR. \nonumber
\end{eqnarray}
The two linear constraints in AP enforce one-to-one correspondence between the features in $\subR$ and $\subL$, see Fig.~\ref{fig:ModlingOcclusion}(a).

For any fixed matching $\M$ the appearance term $ \sum_{p \in \subL}\sum_{q \in \subR} Q(p,q)\;x_{pq}$ in AP's objective function becomes constant. After finding an optimal $\M$ for AP, we could further decrease the objective value by re-estimating homography $\param_h$ minimizing the geometric error, e.g. see first term in~\eqref{eq:symmatchingscore}, over all the currently matched features. We can continue to iteratively re-estimate matching $\M$ and homography $\param_h$ until the objective value of AP could not be reduced any more.

The described optimization procedure maybe sensitive to the initial homography found by \RANSAC. In an effort to reduce such sensitivity we repeat the whole procedure several times, and report as ground truth matching $\M$ and model $\param_h$ that have the lowest value of the objective function of AP.

Now we can discuss possible occlusions that we ignored so far. The presence of occlusions or outliers introduces two problems. First, the number of features in corresponding regions are no longer guaranteed to be the same. Such an imbalance between the features has to be addressed in order to enforce one-to-one correspondence. Second, we can no longer assume that there exists one homography that fits all features.

\begin{figure}
\begin{center}
\begin{tabular}{p{0.45\textwidth}p{0.45\textwidth}}
\multicolumn{1}{c}{\includegraphics[height=0.3\textheight]{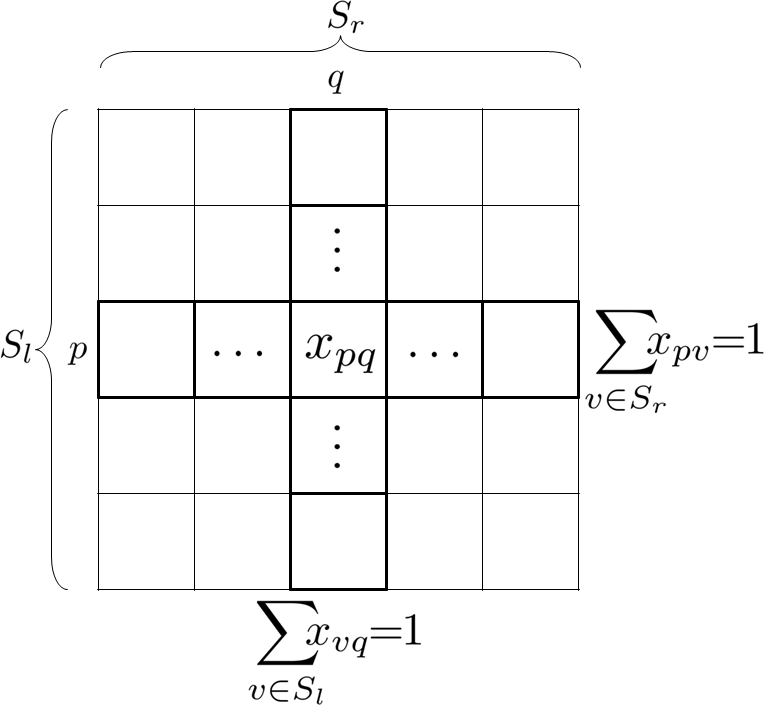}}&
\multicolumn{1}{c}{\includegraphics[height=0.3\textheight]{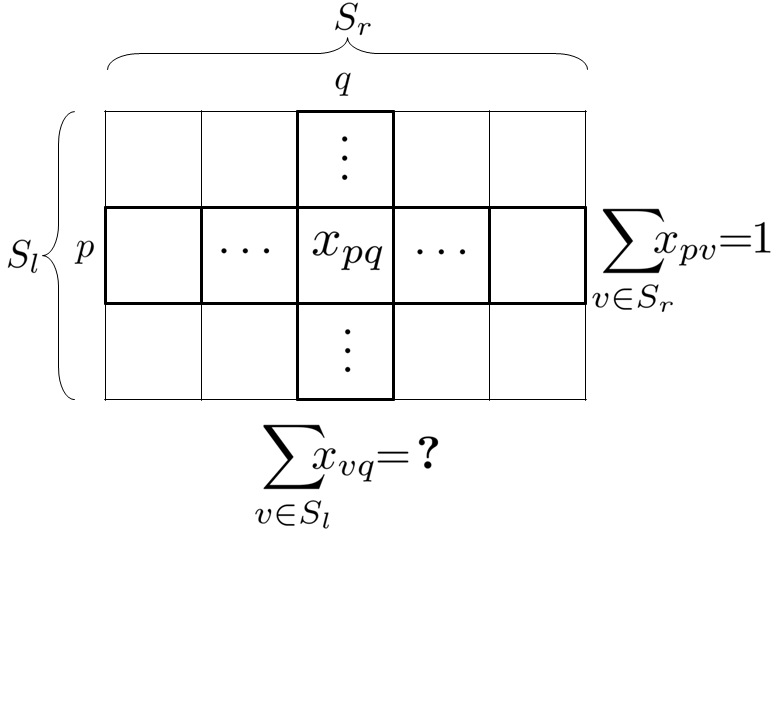}}\\
\multicolumn{1}{c}{{\small (a) No occlusions.}} &
\multicolumn{1}{c}{{\small(b) Imbalance between $\subL$ and $\subR$.}}\\
\multicolumn{1}{c}{\includegraphics[height=0.3\textheight]{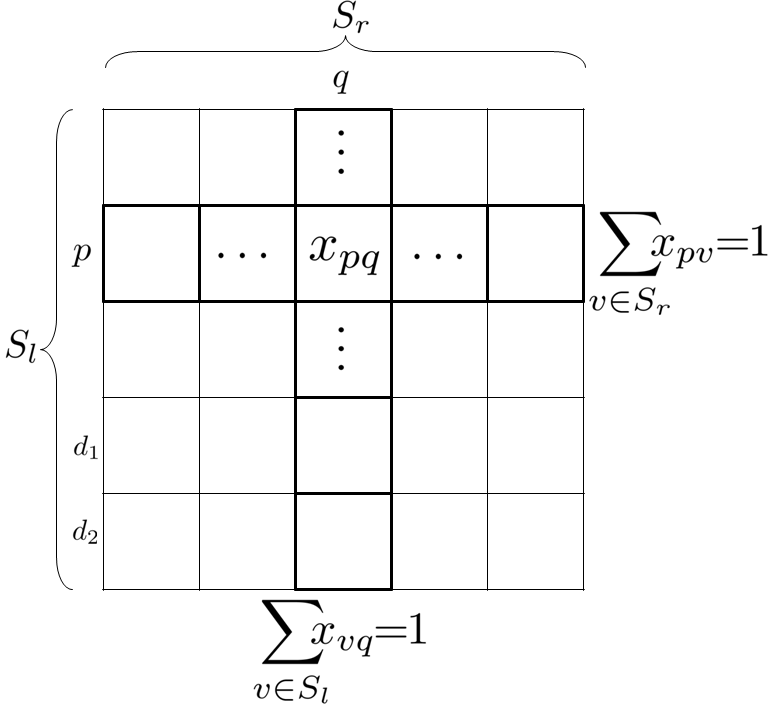}}&
\multicolumn{1}{c}{\includegraphics[height=0.3\textheight]{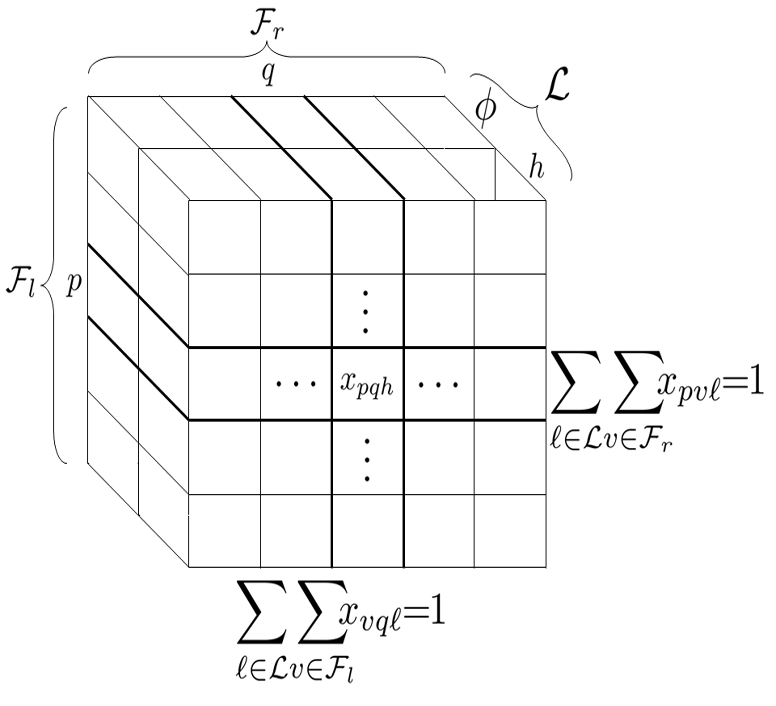}}\\
\multicolumn{1}{c}{{\small (c) Balancing with dummy features.}} &
\multicolumn{1}{c}{{\small (d) GAP with occlusion  model $\Phi$.}}
\end{tabular}
\end{center}
\vspace{-2em}
\caption{{\small Figure (a) shows the straight forward case, no occlusions, for enforcing the one-to-one correspondences. Notice that in this case the number of features in both images is the same and therefore the one-to-one correspondences constraints are balanced. Figure (b) shows a case with unbalanced one-to-one correspondence constraints, i.e there are no enough features in $S_l$ to enforce the one-to-one correspondence constraints. And, figure (c) shows how to balanced the one-to-one correspondence constraints by introducing the dummy feature $d$. Figure (d) shows how to account for occlusions, using occlusion model $\phi$, in case of balanced one-to-one correspondence constraints.}}
\label{fig:ModlingOcclusion}
\end{figure}

To balance out any possible difference between the sizes of sets $\subL$ and $\subR$ we use dummy features with a constant matching cost penalty, as detailed below. Without loss of generality we can assume that the number of extracted features in the left region is less than or equal to those in the right region\footnote{We could always swap the regions to satisfy that assumption as long as the used geometric error and appearance measures are symmetric.}. In this case, there are at least $|\subR|-|\subL|$ occluded features in the right image. As illustrated in Fig.~\ref{fig:ModlingOcclusion}(b), an imbalance between the number of features renders the one-to-one correspondence constraints in AP infeasible. One way to overcome this problem is to add $|\subR|-|\subL|$ dummy features to set $\subL$, see Fig.~\ref{fig:ModlingOcclusion}(c). We define a fixed matching cost penalty $\unary_{dq}=T$ for assigning any dummy feature $d$ to any $q$ in $\subR$. It is also possible to use only one dummy feature $d$ in $\subL$ but for that specific feature constraint $\sum_{q\in \subR} x_{dq}=1$ would have to be replaced by $\sum_{q \in \subR} x_{dq}=|\subR| - |\subL\backslash\{d\}|$. We adapt the first approach with multiple dummy features only to simplify our notation and avoid the special handling of feature $d$. The use of dummy feature/entity is a common technique for balancing out unbalanced assignment problems in operations research \cite{OR:rao05}.

Even under the assumption that $|\subL|=|\subR|$ occlusions are possible and we can not assume that there exists a homography that fits all features. In order to make our approach robust to occlusions/outliers we use generalized the assignment problem, GAP, to allow each feature to choose between two models: a homography $\param_h$ and an occlusion model $\phi$ such that $\unary_{pq}(\phi)=T$ for any $ p \in \subL$ and $q \in \subR$. The use of an occlusion (or outlier) model with a uniformly distributed matching cost is a fairly common technique in Computer Vision \cite{LabelCosts:IJCV12,PEARL:IJCV12}, see Fig.~\ref{fig:ModlingOcclusion}(d).

\section{Evaluation}
\label{sc:evaluation}
In this section, we compare the matching quality of the EFM framework vs. standard SIFT matching \cite{sift:04}. Then we discuss some of the EFM framework properties, e.g. convergence rate and the effect of the initial set of proposals size on the matching quality. Finally, we compare the quality of the estimated models by the EFM framework to the models estimated by an EF algorithm~\PEARL \cite{LabelCosts:IJCV12}.

Our matching evaluation criterion is based on Receiver Operating Characteristics (ROC) of the True Positive Rate (TPR) vs.~the False Positive Rate (FPR). The ROC attributes for matching $\M$ and ground truth (GT) matching $\M_{GT}$ are defined as follows:
\begin{description*}
\item[{\small Positive}] (P) number of matches identified by $\M_{GT}$
\item[{\small Negative}] (N) number of potential matches that were rejected by $\M_{GT}$, \newline i.e.~$N=|\lF|\times|\rF|-P$
\item[{\small True Positive}] (TP) number of matches identified by $\M$ and $\M_{GT}$ (intersection)
\item[{\small False Positive}] (FP) number of matches identified by $\M$ but were rejected by $\M_{GT}$
\item[{\small True Positive Rate}] (TPR) $\frac{TP}{P}$
\item[{\small False Positive Rate}] (FPR) $\frac{FP}{N}$.
\end{description*}

Figure~\ref{fig:ROCcurve}(a) shows the ROC curve of standard SIFT matching achieved by varying the second best ratio (SBR) threshold where SBR is the ratio of distance between a left feature descriptor and the closest right features descriptor to the distance of the second closest. EFM is non-deterministic and the energy of convergence, a.k.a final energy, depends on the size of initial set of proposals $|\labelset|$. Therefore, for EFM we show a scatter plot that relates the ROC attributes to the final energy (colour coded) by varying $|\labelset|$. As can be see, EFM outperforms standard SIFT matching and the lower the final energy the better the matching quality. Furthermore, Fig.~\ref{fig:ROCcurve}(b) shows multiple histograms relating the final energy frequencies, of 50 runs, to $|\labelset|$ (colour coded). As can be seen, the bigger $|\labelset|$ is the more likely the final energy is going to be small and the more likely that EFM behaviour becomes more deterministic over different runs.

Figure~\ref{fig:EFMConvergence} shows the effect of EFM iterations on the energy with respect to time for different $|\labelset|$. For each $|\labelset|$ the experiment is repeated 50 times. On the average each iteration took 1 min., and most of the energy was reduced in the first three iterations. EFM converged on the average after 5 iterations. The plots in Fig.~\ref{fig:ROCcurve} and~\ref{fig:EFMConvergence} are shown for the Merton College example, in Fig.~\ref{fig:metronCollgeZoomed}, to illustrate the general characteristics/behaviour of our method. It will be meaningless to average these plots over multiple examples since they would not share the same energy scale, i.e. a low energy for one example could be high for another one.

\begin{figure}[h]
\begin{center}
\begin{tabular}{p{0.45\textwidth}p{0.45\textwidth}}
\includegraphics[height=0.25\textheight]{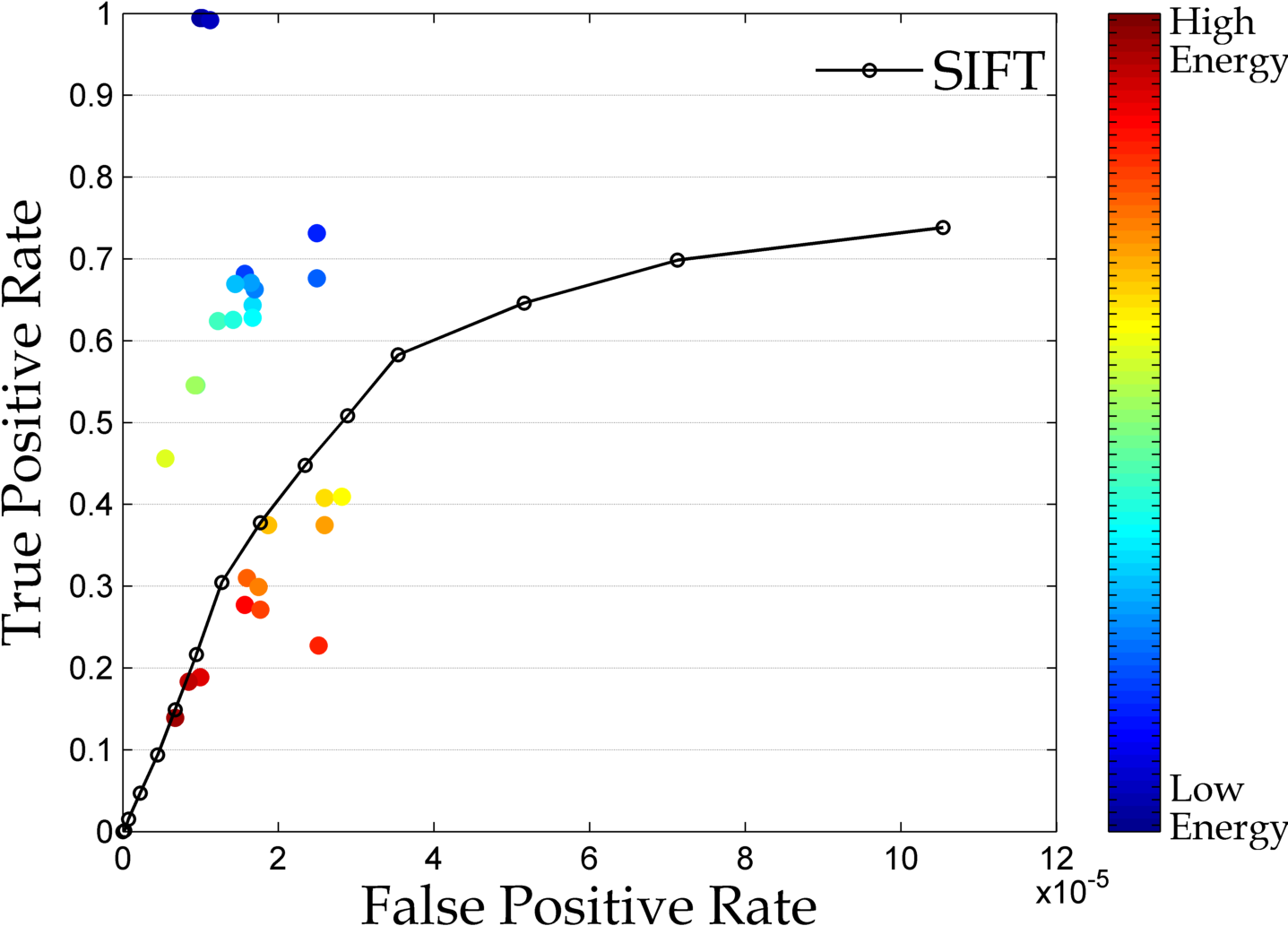}\vspace{1em}&
\includegraphics[height=0.25\textheight]{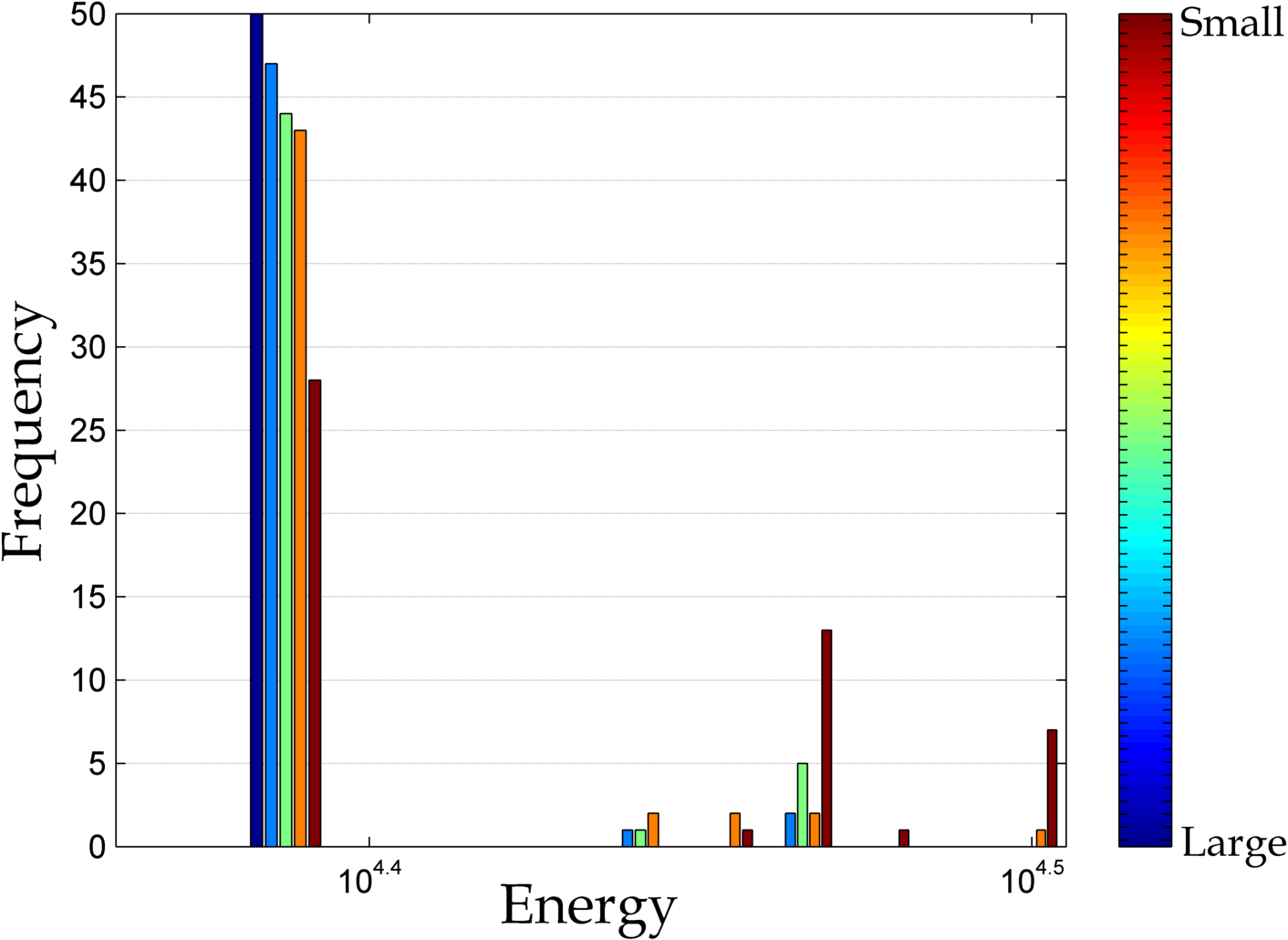}\\
{\small (a) EFM vs SIFT matching quality} & {\small (b) Effect of $|\labelset|$ on the final energy}\\
\end{tabular}
\end{center}
\vspace{-1em}
\caption{{\small Figure~(a) shows the ROC curve of the standard SIFT matches by varying the SBR threshold, and the scatter plot represents EFM results for different sizes of initial set or proposals. The scatter plot is colour coded to show relation between the achieved final energy and the quality of the matching, the lower the energy (blue) the better the matching. Figure~(b) shows multiple histograms of the final energies for different sizes of initial set of proposals---blue indicates a large initial set of proposals while red indicates a small set. The larger the size of the initial set of proposals the more likely that EFM will converge to a low energy.}}
\label{fig:ROCcurve}
\end{figure}

\begin{figure}[h]
\begin{center}
\includegraphics[height=0.35\textheight,width=0.65\textwidth]{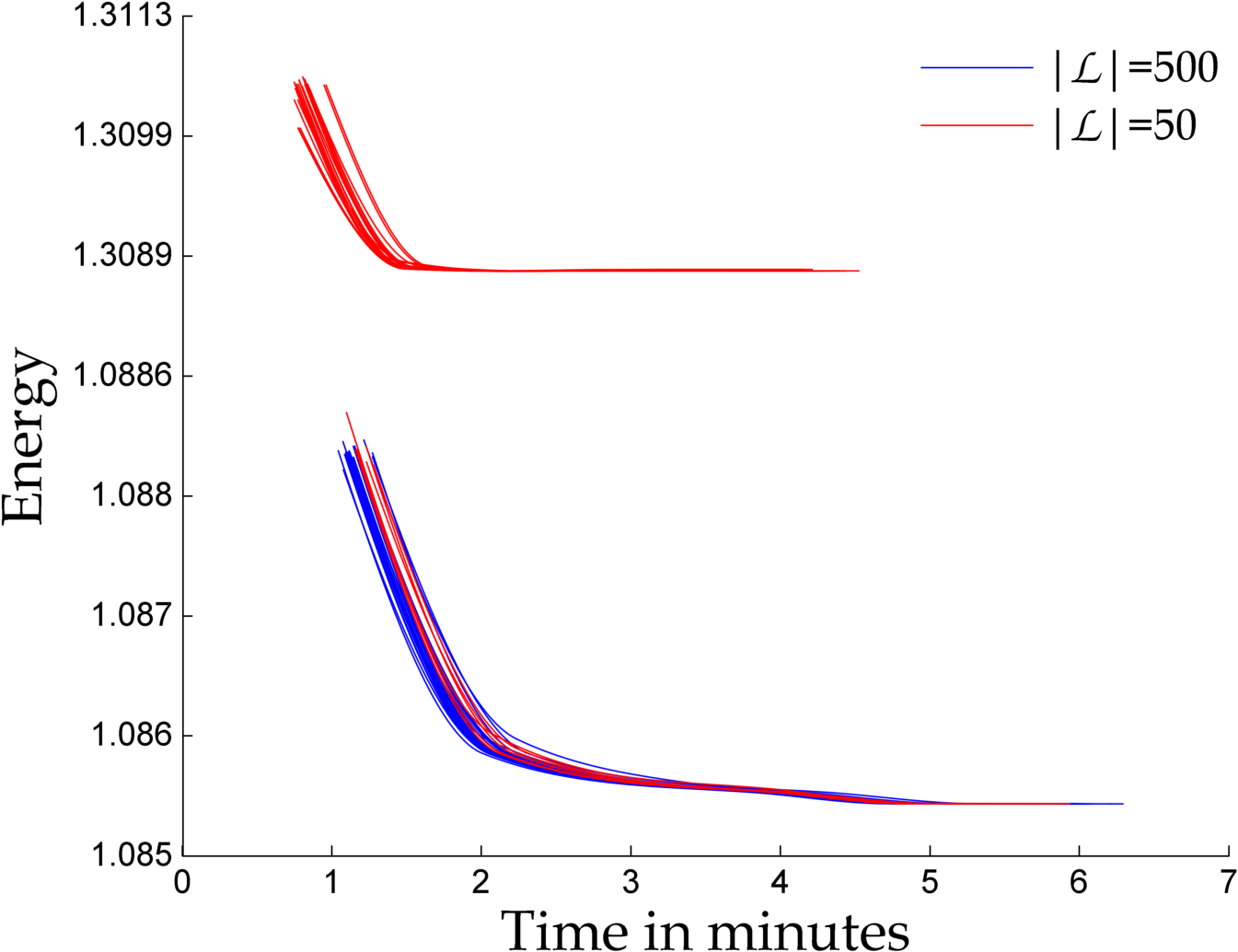}
\end{center}
\vspace{-1em}
\caption{{\small EFM energy over time in minutes. EFM converges on the average in 5 iterations, and an iteration on the average takes 1 minutes.}}
\label{fig:EFMConvergence}
\end{figure}

Figures~\ref{fig:metronCollgeZoomed}(a) and (b) show left features of EF \cite{LabelCosts:IJCV12} inliers  and left features of matches identified by EFM, respectively. Also, outliers or unmatched features are shown as x. EFM on the average found double the number of matches compared to using EF and SIFT standard matching. Figures~\ref{fig:metronCollgeZoomed}(c) and (d), and (e) and (f) are the zoom in for Segment 1 and Segment 2 in (a) and (b), respectively. Figures~\ref{fig:metronCollgeZoomed}(g) and (h) show the matching, over a small region, between the left and right images results of EF and EFM, respectively.

Figure~\ref{fig:multipleExamples} shows more results comparing EFM vs. EF and SIFT standard matching. In general EFM was able to find more matches than EF but EFM outperformed EF in two particular examples; the graphite example, shown in second row, in which large viewpoint between left and right images resulted in SIFT standard matching producing only 76 potential matches, and redbrick house example, shown in third row, in which repetitive texture of the bricks reduced the discriminate power of SIFT descriptor.

In order to evaluate the quality of the estimated model $\theta_h$, we will use the following geometric error ratio $GQ(\theta_h)$
$$GQ(\theta_h) := \frac{STE(\theta_h,\labelvars_{GT},\M_{GT})}{STE(\theta_{GT},\labelvars_{GT},\M_{GT})}$$ where $\labelvars_{GT}$ is the ground truth labelling and $STE(\param_h,\labelvar,\M)$ is the Symmetric Transfer Error of $\param_h$, i.e.~geometric error, computed for labelling $\labelvars$ and matching $\M$---the close $GQ(\theta_h)$ is to $1$ the better the model estimate.

Table~\ref{tbl:SBR} shows the effect of increasing viewpoint angle, between the left and right images, on the quality of model estimates. As the viewpoint increases the number of matched points by EF sharply decreases while for EFM the decrease is not as steep. In addition, EF becomes more sensitive to the used SBR threshold for increasing viewpoint, see variance for large
viewpoint. Furthermore, EFM archives near optimal matching, see TRP and FPR.

The used fitting threshold $T$ affects the ground truth, EFM and EF results, as it is a parameter for these methods. Table~\ref{tbl:thresholdEffect} shows the effect of increasing $T$ on EF and EFM. For the case of $T\leq1$, $T$ is underestimated and running the ground truth multiple times will result in similar final energies but slightly different matching. The more we decrease $T$ the more different the matchings will be. For $T\leq2$ and $T\leq3$ the ground truth result become more deterministic over multiple runs. Finally, when computing the ground truth for all the examples shown above we manually handed tuned $T$ to find the smallest $T$ that gives a stable ground truth over multiple runs.

\begin{figure}
\begin{center}
\begin{tabular}{p{0.45\textwidth}p{0.45\textwidth}}
\multicolumn{1}{c}{\includegraphics[height=0.15\textheight,width=0.40\textwidth]{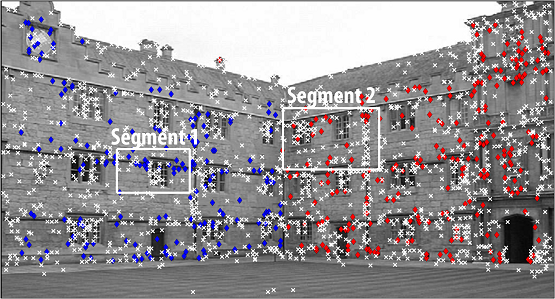}} &
\multicolumn{1}{c}{\includegraphics[height=0.15\textheight,width=0.40\textwidth]{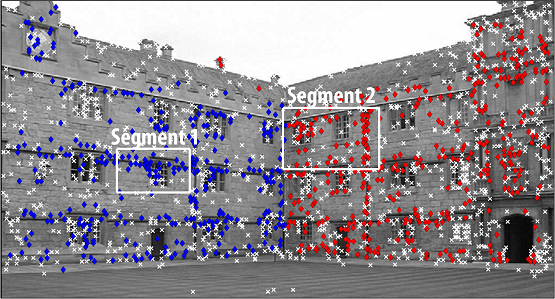}} \\
\multicolumn{1}{c}{{\small (a) EF left image result}} &
\multicolumn{1}{c}{{\small (b) EFM left image result}}\\
\multicolumn{1}{c}{\includegraphics[height=0.15\textheight]{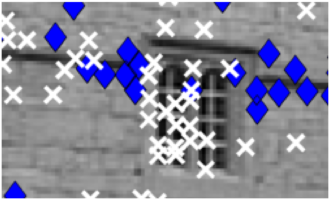}} &
\multicolumn{1}{c}{\includegraphics[height=0.15\textheight]{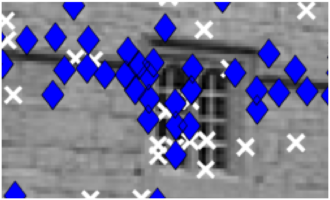}} \\
\multicolumn{1}{c}{{\small (c) Enlarged Segment 1 in (a)}} &
\multicolumn{1}{c}{{\small (d) Enlarged Segment 1 in (b)}}\\
\multicolumn{1}{c}{\includegraphics[height=0.15\textheight]{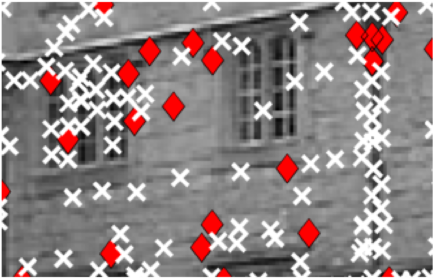}} &
\multicolumn{1}{c}{\includegraphics[height=0.15\textheight]{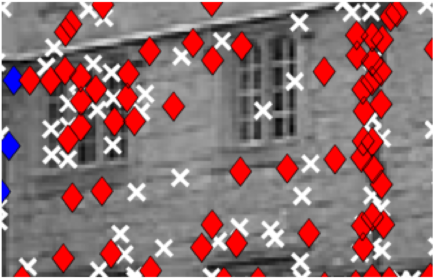}} \\
\multicolumn{1}{c}{{\small (e) Enlarged Segment 2 in (a)}} &
\multicolumn{1}{c}{{\small (f) Enlarged Segment 2 in (b)}}\\
\multicolumn{1}{c}{\includegraphics[height=0.20\textheight,width=0.30\textwidth]{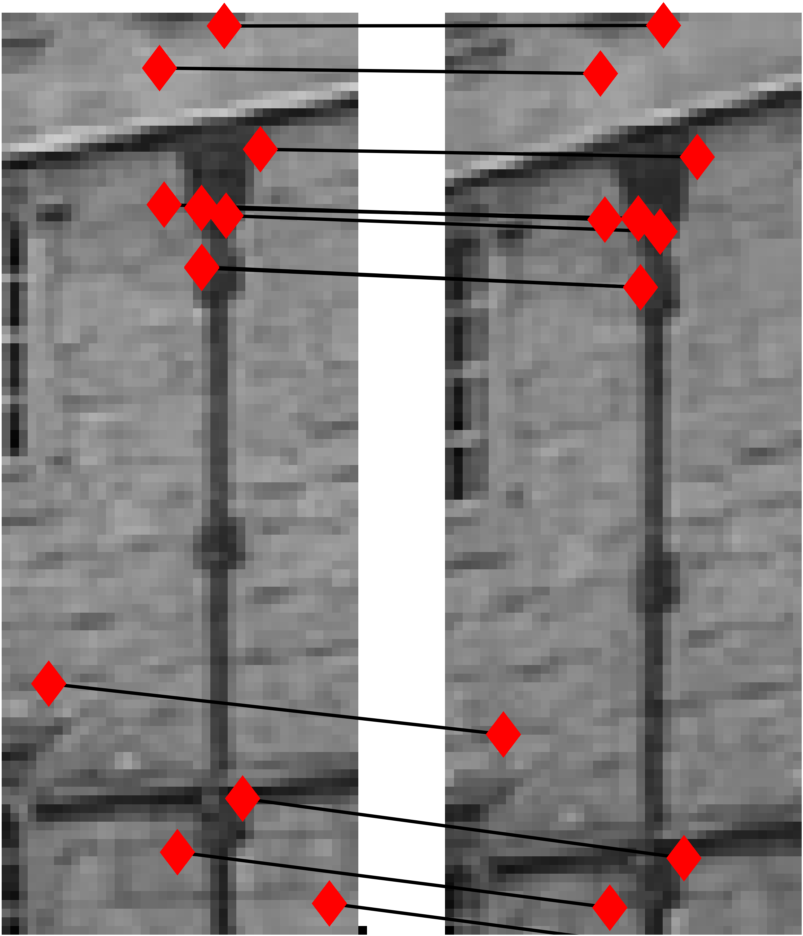}} &
\multicolumn{1}{c}{\includegraphics[height=0.20\textheight,width=0.30\textwidth]{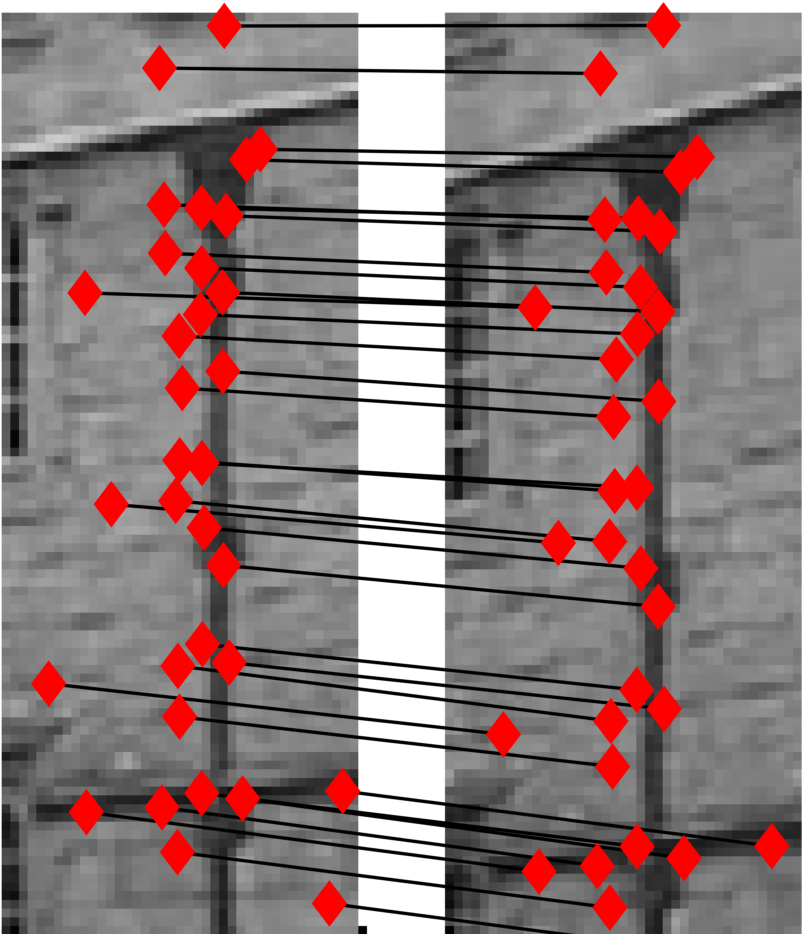}}\\
{\small (g) Part of the EF matching, between left and right images, i.e.~inliers of the SIFT standard matching.}  &
{\small (h) Part of the EFM matching, between left and right images.}
\end{tabular}
\end{center}
\caption{{\small Merton College from VGG Oxford, Fig.~(a) shows EF result (average TPR=0.51 and FPR=1.6E-05) and (b) shows EFM result (average TPR=0.98 and FPR=9.1E-06). The averaging is done over 50 runs. Figures~(c-d) show the enlargement of Segment 1 in (a) and (b), respectively, and Fig.~(e-f) show the enlargement of Segment 2 in (a) and (b), respectively. Figures~(g-h) show the matching, between two small regions in the left and right images, of the EF and EFM results, respectively.  The average GQ ratios are 1.042 and 1.0630 for
EF estimated models, and 1.0102 and 1.0079 for EFM
estimated models.}}
\label{fig:metronCollgeZoomed}
\end{figure}

\begin{figure}
\begin{center}
\begin{tabular}{p{0.33\textwidth}p{0.33\textwidth}p{0.33\textwidth}}
\includegraphics[height=0.15\textheight,width=0.30\textwidth]{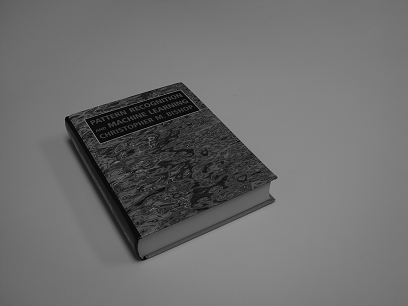} &
\includegraphics[height=0.15\textheight,width=0.30\textwidth]{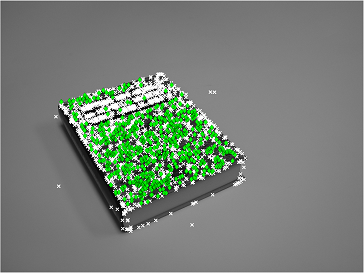} &
\includegraphics[height=0.15\textheight,width=0.30\textwidth]{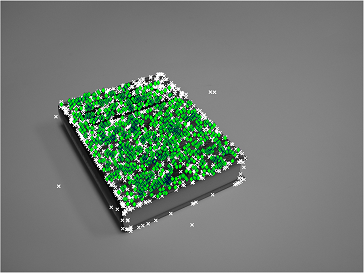} \\
\includegraphics[height=0.15\textheight,width=0.30\textwidth]{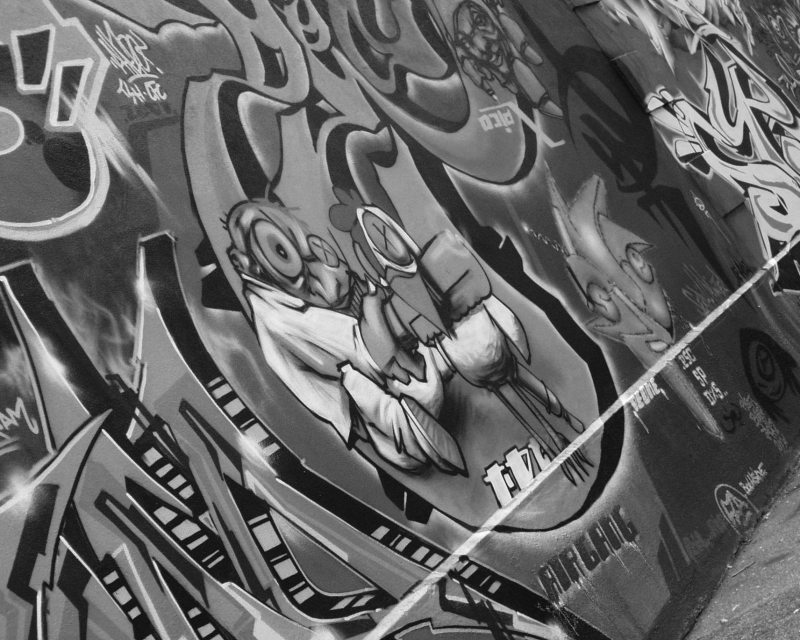} &
\includegraphics[height=0.15\textheight,width=0.30\textwidth]{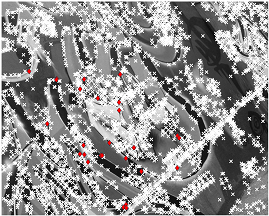} &
\includegraphics[height=0.15\textheight,width=0.30\textwidth]{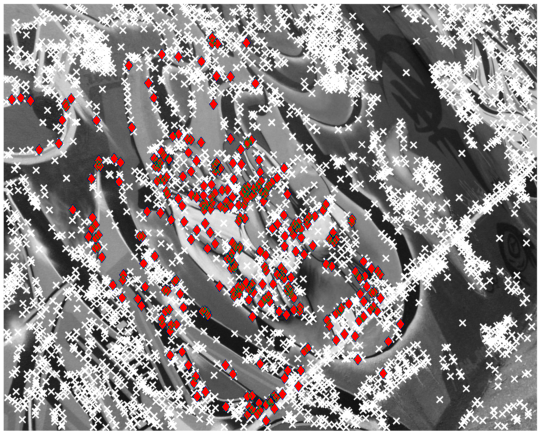} \\
\includegraphics[height=0.15\textheight,width=0.30\textwidth]{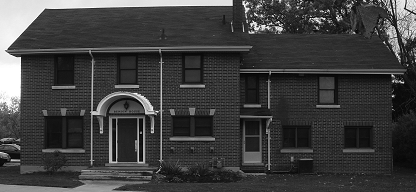} &
\includegraphics[height=0.15\textheight,width=0.30\textwidth]{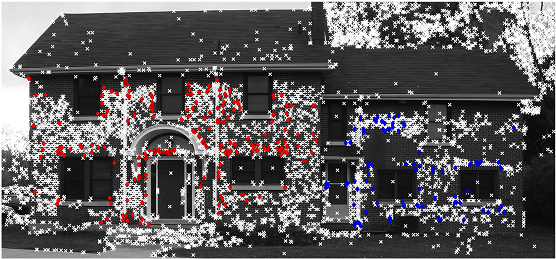} &
\includegraphics[height=0.15\textheight,width=0.30\textwidth]{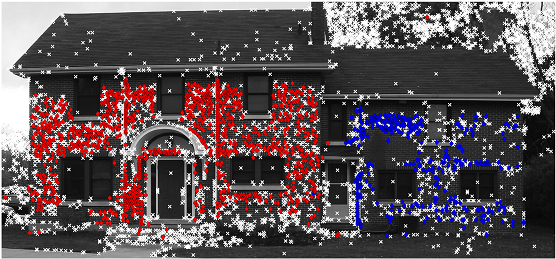} \\
\includegraphics[height=0.15\textheight,width=0.30\textwidth]{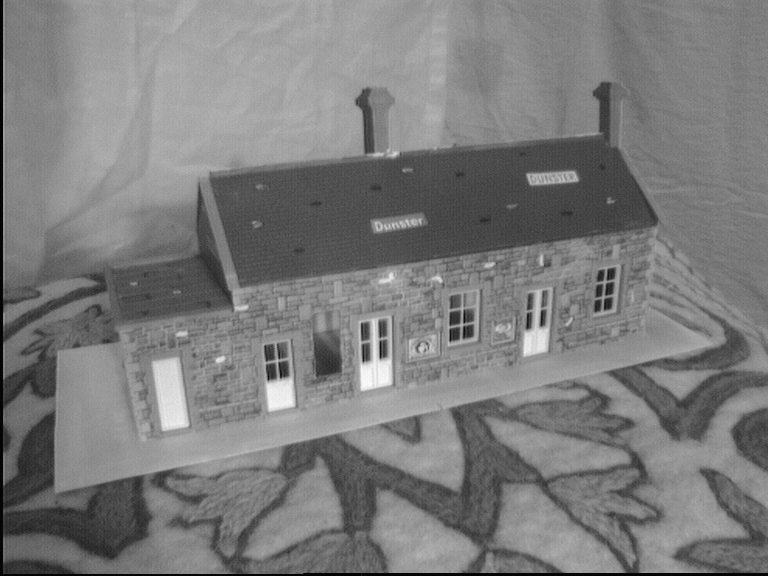} &
\includegraphics[height=0.15\textheight,width=0.30\textwidth]{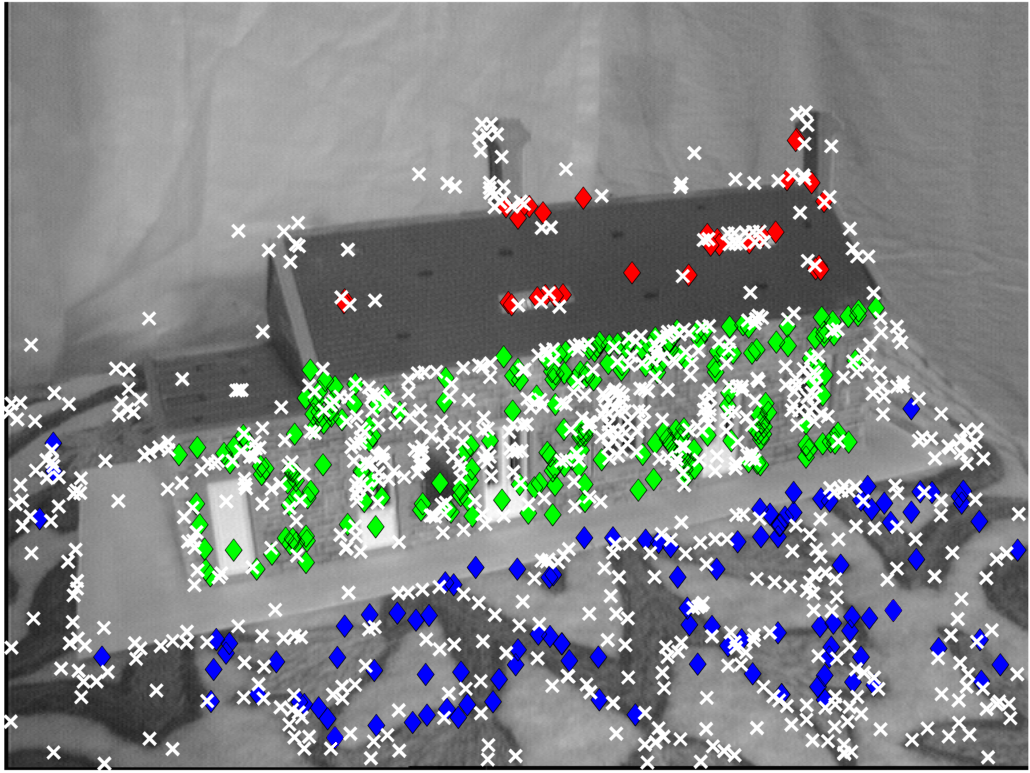} &
\includegraphics[height=0.15\textheight,width=0.30\textwidth]{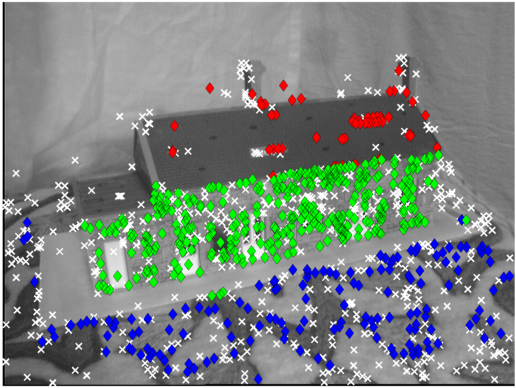} \\
\includegraphics[height=0.15\textheight,width=0.30\textwidth]{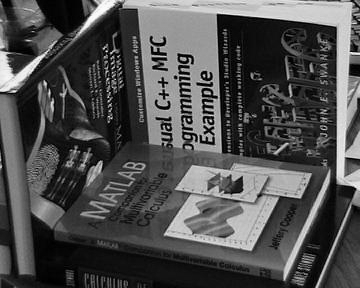} &
\includegraphics[height=0.15\textheight,width=0.30\textwidth]{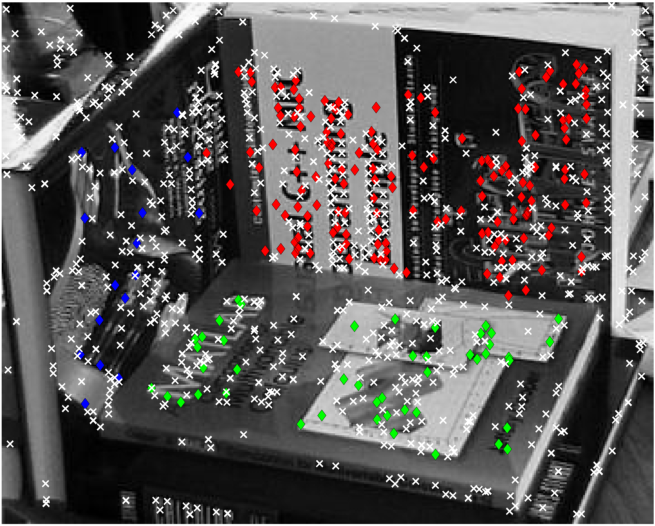} &
\includegraphics[height=0.15\textheight,width=0.30\textwidth]{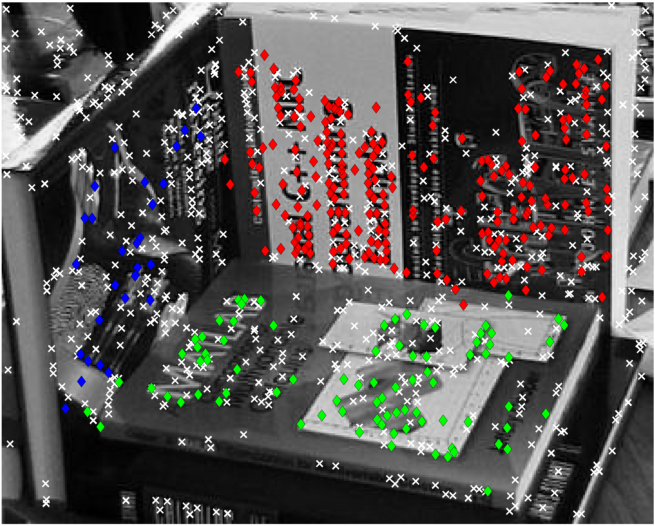} \\
\includegraphics[height=0.15\textheight,width=0.30\textwidth]{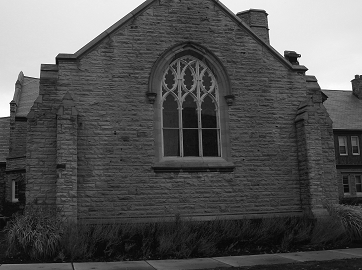} &
\includegraphics[height=0.15\textheight,width=0.30\textwidth]{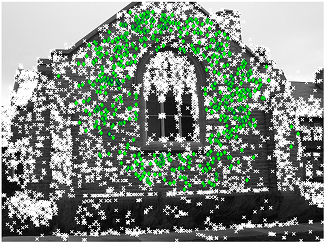} &
\includegraphics[height=0.15\textheight,width=0.30\textwidth]{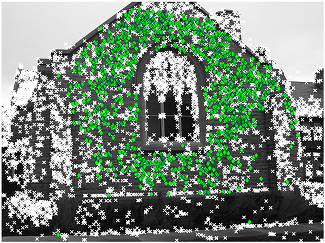} \\
\end{tabular}
\end{center}
\caption{{\small First column shows left images of the examples, second and third columns show the EF and EFM results, respectively. The average increase, over 50 runs, in the number of matches found by EFM in comparison to EF for the examples shown above (top to bottom) is $0.76$, $10.53$, $3.33$, $0.44$, $0.6$ and $0.68$, respectively.}}
\label{fig:multipleExamples}
\end{figure}

\begin{table}
\begin{center}
\begin{tabular}{c c|r|r|r|r|r|r|r|}
\cline{3-9}
           &            & \multicolumn{3}{c|}{{\bf GQ}} &                  \multicolumn{4}{c|}{{\bf ROC}} \\
\cline{3-9}
           &            & {\bf median} & {\bf mean} & {\bf variance} & {\bf TP } & {\bf FP } & {\bf TPR } & {\bf FPR } \\
\cline{2-9}
\multicolumn{1}{c|}{\multirow{8}{*}{\begin{sideways}small viewpoint\end{sideways}}} &         \multirow{2}{*}{EFM} & \multirow{2}{*}{\bf 1.0048} &     \multirow{2}{*}{1.0074} &   \multirow{2}{*}{4.00E-06} &        \multirow{2}{*}{824} &      \multirow{2}{*}{18.08} &       \multirow{2}{*}{\bf{0.98}} &   \multirow{2}{*}{\bf{2.30E-06}} \\
\multicolumn{1}{c|}{}&  & & & & & & &    \\
\cline{2-9}
\multicolumn{1}{c|}{}& {EF} & \multirow{2}{*}{\bf 1.0386} &    \multirow{2}{*}{1.0475} &   \multirow{2}{*}{1.00E-03} &        \multirow{2}{*}{602} &      \multirow{2}{*}{31.66} &       \multirow{2}{*}{0.72} &   \multirow{2}{*}{4.10E-06} \\
\multicolumn{1}{c|}{}& {SBR=0.6} &  & &  & &  & & \\
\cline{2-9}
\multicolumn{1}{c|}{}& {EF} &     \multirow{2}{*}{1.0415} &     \multirow{2}{*}{1.0519} &   \multirow{2}{*}{1.30E-03} &       \multirow{2}{*}{652} &      \multirow{2}{*}{41.14} &       \multirow{2}{*}{0.78} &   \multirow{2}{*}{5.20E-06} \\
\multicolumn{1}{c|}{}& {SBR=0.7} &  & &  & &  & & \\
\cline{2-9}
\multicolumn{1}{c|}{}& {EF} &      \multirow{2}{*}{1.0460} &     \multirow{2}{*}{1.0521} &   \multirow{2}{*}{8.00E-04} &        \multirow{2}{*}{691} &       \multirow{2}{*}{51.60} &    \multirow{2}{*}{0.82} &   \multirow{2}{*}{6.60E-06} \\
\multicolumn{1}{c|}{}& {SBR=0.8} &  & &  & &  & & \\
%\cmidrule{2-9}\morecmidrules\cmidrule{2-9}
\cline{2-9} \\
\cline{2-9}
\multicolumn{1}{c|}{\multirow{8}{*}{\begin{sideways}medium viewpoint\end{sideways}}} &         \multirow{2}{*}{EFM} & \multirow{2}{*}{\bf 1.0183} &     \multirow{2}{*}{1.0194} &   \multirow{2}{*}{1.00E-06} &        \multirow{2}{*}{501} &      \multirow{2}{*}{26.96} &       \multirow{2}{*}{\bf{0.97}} &   \multirow{2}{*}{3.10E-06} \\
\multicolumn{1}{c|}{}&  & & & & & & &    \\
\cline{2-9}
\multicolumn{1}{c|}{}& {EF} & \multirow{2}{*}{1.1742} &    \multirow{2}{*}{1.3031} &   \multirow{2}{*}{1.71E-01} &        \multirow{2}{*}{94} &      \multirow{2}{*}{19.18} &       \multirow{2}{*}{0.18} &   \multirow{2}{*}{\bf{2.20E-06}} \\
\multicolumn{1}{c|}{}& {SBR=0.6} &  & &  & &  & & \\
\cline{2-9}
\multicolumn{1}{c|}{}& {EF} &     \multirow{2}{*}{1.1989} &     \multirow{2}{*}{1.3012} &   \multirow{2}{*}{8.64E-02} &       \multirow{2}{*}{171} &      \multirow{2}{*}{33.22} &       \multirow{2}{*}{0.33} &   \multirow{2}{*}{3.80E-06} \\
\multicolumn{1}{c|}{}& {SBR=0.7} &  & &  & &  & & \\
\cline{2-9}
\multicolumn{1}{c|}{}& {EF} &      \multirow{2}{*}{\bf 1.0806} &     \multirow{2}{*}{1.2594} &   \multirow{2}{*}{1.09E-00} &        \multirow{2}{*}{256} &       \multirow{2}{*}{49.20} &    \multirow{2}{*}{0.49} &   \multirow{2}{*}{5.6E-06} \\
\multicolumn{1}{c|}{}& {SBR=0.8} &  & &  & &  & & \\
\cline{2-9} \\
\cline{2-9}
\multicolumn{1}{c|}{\multirow{8}{*}{\begin{sideways}large viewpoint\end{sideways}}} &         \multirow{2}{*}{EFM} & \multirow{2}{*}{\bf 1.0523} &     \multirow{2}{*}{1.0698} &   \multirow{2}{*}{2.00E-03} &        \multirow{2}{*}{300} &      \multirow{2}{*}{15.72} &       \multirow{2}{*}{\bf{0.96}} &   \multirow{2}{*}{1.70E-06} \\
\multicolumn{1}{c|}{}&  & & & & & & &    \\
\cline{2-9}
\multicolumn{1}{c|}{}& {EF} & \multirow{2}{*}{2.6412} &    \multirow{2}{*}{2.6413} &   \multirow{2}{*}{1.30E-06} &        \multirow{2}{*}{9} &      \multirow{2}{*}{2}&       \multirow{2}{*}{0.03}&   \multirow{2}{*}{\bf{2.20E-07}} \\
\multicolumn{1}{c|}{}& {SBR=0.6} &  & &  & &  & & \\
\cline{2-9}
\multicolumn{1}{c|}{}& {EF} &     \multirow{2}{*}{\bf 1.8993} &     \multirow{2}{*}{2.2440} &   \multirow{2}{*}{1.22E-00} &       \multirow{2}{*}{19} &      \multirow{2}{*}{5.48} &       \multirow{2}{*}{0.06} &   \multirow{2}{*}{5.90E-07} \\
\multicolumn{1}{c|}{}& {SBR=0.7} &  & &  & &  & & \\
\cline{2-9}
\multicolumn{1}{c|}{}& {EF} &      \multirow{2}{*}{2.4915} &     \multirow{2}{*}{3.8799} &   \multirow{2}{*}{9.31E-00} &        \multirow{2}{*}{36} &       \multirow{2}{*}{13.04} &    \multirow{2}{*}{0.12} &   \multirow{2}{*}{1.40E-06} \\
\multicolumn{1}{c|}{}& {SBR=0.8} &  & &  & &  & & \\
\cline{2-9} \\
\cline{2-9}
\multicolumn{1}{c|}{\multirow{8}{*}{\begin{sideways}multi-model case\end{sideways}}} &   \multirow{2}{*}{EFM} & {\bf 1.0102} &     1.0102 &  1.90E-09 & \multirow{2}{*}{656} & \multirow{2}{*}{36.49}  & \multirow{2}{*}{\bf{0.98}} & \multirow{2}{*}{\bf{9.10E-06}} \\
\cline{3-5}
\multicolumn{1}{c|}{}& &{\bf 1.0046} & 1.0079& 1.20E-05& & & &     \\
\cline{2-9}
\multicolumn{1}{c|}{}& {EF} &  1.0625 &    1.0681 &   7.00E-04 &  \multirow{2}{*}{258} &  \multirow{2}{*}{45.48} & \multirow{2}{*}{0.38} & \multirow{2}{*}{1.10E-05} \\
\cline{3-5}
\multicolumn{1}{c|}{}&  {SBR=0.6}&1.0397 &1.0514 & 1.80E-03& & & &  \\
\cline{2-9}
\multicolumn{1}{c|}{}& {EF} &  {\bf 1.0383}&    1.0420 &   2.00E-04 &  \multirow{2}{*}{344} &  \multirow{2}{*}{63.04} & \multirow{2}{*}{0.52} & \multirow{2}{*}{1.60E-05} \\
\cline{3-5}
\multicolumn{1}{c|}{}&  {SBR=0.7}&1.0427 & 1.0630& 2.20E-03& & & &    \\
\cline{2-9}
\multicolumn{1}{c|}{}& {EF} &  {\bf 1.0218}&    1.0243 &   3.00E-04 &  \multirow{2}{*}{431} &  \multirow{2}{*}{76.18} & \multirow{2}{*}{0.64} & \multirow{2}{*}{1.90E-05} \\
\cline{3-5}
\multicolumn{1}{c|}{}&  {SBR=0.8}&1.0447 &1.0905 &9.80E-03 & & & &    \\
\cline{2-9}
\end{tabular}
\caption{{\small In the case of a single model and increasing viewpoint (Graphite VGG Oxford), the first three blocks show the average, over 50 runs, ROC attributes and $GQ$ of EFM and EF (using different SBR ratios). The EFM and EF results were comparable for small viewpoint but as the viewpoint increases EFM model estimates becomes more reliable in comparison to EF estimates. The last block shows the results for a  multi-model case (Merton College VGG Oxford). In both cases, EFM achieved near optimal matching.}}
\label{tbl:SBR}
\end{center}
\end{table}

\begin{table}
\begin{center}
\begin{tabular}{c c|r|r|r|r|r|r|r|}
\cline{3-9}
           &            & \multicolumn{3}{c|}{{\bf GQ}} &                  \multicolumn{4}{c|}{{\bf ROC}} \\
\cline{3-9}
           &            & {\bf median} & {\bf mean} & {\bf variance} & {\bf TP } & {\bf FP } & {\bf TPR } & {\bf FPR } \\
\cline{2-9}
\multicolumn{1}{c|}{\multirow{4}{*}{\begin{sideways}$T \leq 1$\end{sideways}}} &   \multirow{2}{*}{EFM} & 1.0031& 1.0040  & 6.3E-6  & \multirow{2}{*}{529} & \multirow{2}{*}{34.20}  & \multirow{2}{*}{0.95} & \multirow{2}{*}{8.5E-06} \\
\cline{3-5}
\multicolumn{1}{c|}{}& &1.0225 & 1.1251 & 0.0207 & & & &   \\
\cline{2-9}
\multicolumn{1}{c|}{}& {EF} & 1.1999 & 1.2470  & 0.0427 &  \multirow{2}{*}{241} &  \multirow{2}{*}{33.12} & \multirow{2}{*}{0.43} & \multirow{2}{*}{8.3E-06} \\
\cline{3-5}
\multicolumn{1}{c|}{}&  {SBR=0.7}& 1.2384 & 1.2750 & 0.0331 & & & &    \\
\cline{2-9} \\
\cline{2-9}
\multicolumn{1}{c|}{\multirow{4}{*}{\begin{sideways}$T \leq 2$\end{sideways}}} &   \multirow{2}{*}{EFM} & 1.0102& 1.0102 & 1.9E-9 & \multirow{2}{*}{656} & \multirow{2}{*}{36.489}  & \multirow{2}{*}{0.98} & \multirow{2}{*}{9.1E-06} \\
\cline{3-5}
\multicolumn{1}{c|}{}& & 1.0046& 1.0079& 1.2E-5& & & &     \\
\cline{2-9}
\multicolumn{1}{c|}{}& {EF} & 1.0383 & 1.0427  & 0.0002 &  \multirow{2}{*}{344} &  \multirow{2}{*}{63.04} & \multirow{2}{*}{0.52} & \multirow{2}{*}{1.6E-05} \\
\cline{3-5}
\multicolumn{1}{c|}{}&  {SBR=0.7}& 1.0427& 1.0630& 0.0022& & & &    \\
\cline{2-9} \\
\cline{2-9}
\multicolumn{1}{c|}{\multirow{4}{*}{\begin{sideways}$T \leq 3$\end{sideways}}} &   \multirow{2}{*}{EFM} &1.0084 & 1.0083 & 0.4E-9 & \multirow{2}{*}{720} & \multirow{2}{*}{45}  & \multirow{2}{*}{0.99} & \multirow{2}{*}{1.1E-05} \\
\cline{3-5}
\multicolumn{1}{c|}{}& & 1.0046&1.0079 & 1.2E-5& & & &     \\
\cline{2-9}
\multicolumn{1}{c|}{}& {EF} & 1.0372 &  1.0347 &  2.5E-5&  \multirow{2}{*}{369} &  \multirow{2}{*}{65.34} & \multirow{2}{*}{0.51} & \multirow{2}{*}{1.6E-05} \\
\cline{3-5}
\multicolumn{1}{c|}{}&  {SBR=0.7}& 1.0500 & 1.0589 & 0.0015 & & & &    \\
\cline{2-9}
\end{tabular}
\caption{{\small shows the effect of the fitting threshold $T$ (used in computing ground truth, EF and EFM) on the average quality of estimated models and ROC attributes, over 50 runs.
The performance for both EF and EFM in the case of $T\leq2$ is better than the case of $T\leq1$ because the threshold in the later case is underestimated---see $GQ$ variance.
Furthermore, the average increase in the model estimate quality%, over the two models,
for EF and EFM between $T\leq3$ and $T\leq2$ is not as significant as the average increase between $T\leq2$ and $T\leq1$.}}
\label{tbl:thresholdEffect}
\end{center}
\end{table}
\vspace{-2ex}
\section{Conclusions}

We introduced two energy functionals that use different regularizers for the {\em fit-\&-match} problem. We also introduced optimization frameworks for these functionals. Our experimental results show that our energy-based {\em fit-\&-match} framework finds a near optimal solution for the feature-to-feature matching and better model estimates in contrast to state-of-the-art energy-based fitting frameworks, e.g.~\PEARL. In addition, we showed that for a given set of models it is possible to efficiently find the optimal feature-to-feature matching and match-to-label assignment. Our framework could be used to {\em fit-\&-match} more complex models, e.g.~fundamental matrices, without affecting the framework's complexity.
It could also be used to {\em fit-\&-match} a mixture of different models, e.g.~homographies and affine transformations, and penalize each model/label based on its complexity. Finally, we plan on applying our framework in camera pose estimation.
%Furthermore, our feature-to-feature matching approach, GAP, could be easily extended to handel features with descriptors computed at different scales\footnote{When constructing $\graph^*$ there will be multiple edges between $n_{ph}$ and $n_{qh}$ with edge costs computed using $p$ and $q$ scale specific descriptors.}.

\appendix
\section{Lemma~\ref{lm:mcmf} Proof}
\label{ap:mcmfproof}
The following proof assumes that there exists a feasible solution for GAP with a finite objective value. GAP is unfeasible when $|\lF|\neq|\rF|$, e.g.~$|\lF|<|\rF|$ and in that case adding $(|\rF|-|\lF|)$ dummy features with a fixed matching penalty $T$ to $\lF$ will ensure the GAP feasibility. The objective value of a feasible GAP solution is guaranteed to be finite when GAP is solved over any set of models and an outlier model $\phi$ with a fixed matching penalty $T$ for all possible pairs of matched features.
%Solving a feasible GAP, i.e. $|\lF|=|\rF|$, while including an outlier model will guarantee that the amount of flow needed to saturate $\graph^*$ is $|\lF|$.
We will prove a more general theorem than Lemma~\ref{lm:mcmf}. Lemma~\ref{lm:mcmf} is a derivative of Theorem~\ref{th:gap2mcmf}.

\begin{theorem}
\label{th:gap2mcmf}
``{\it There exists an optimal solution, with an objective value $k^*$, of a GAP instance {\bf  if and only if}
there exists a valid MCMF $\mathbf{F^*}$ over $\graph^*$, of the GAP instance, with $cost(\mathbf{F^*})=k^*$}''.
\end{theorem}
\begin{proof}
Assume that there exists a GAP optimal solution $\M^*_\labelvars$ with an objective value $k^*$.
If there exists a valid MCMF $\mathbf{F}$ over $\graph^*$ with $cost(\mathbf{F})=k$ such that $k<k^*$ then we can construct a feasible GAP solution $\M_\labelvars$ where $\M_\labelvars=\{x_{pqh}=\mathbf{F}(n_{ph},n_{qh}) \;|\;  p\in \lF,q\in \rF, h\in \labelset\}$. Using Corollary~\ref{cr:F_GAP_cost} we can deduce that the objective value of the constructed GAP solution $\M_\labelvars$ is equal to $cost(\mathbf{F})$.
Now we prove that the constructed solution $\M_\labelvars$ is feasible by showing that $\M_\labelvars$ can not be unfeasible, i.e.~one or more of the constraints~\eqref{eq:121Constrains_MF} can not be violated in the constructed solution. Constraints~\eqref{eq:121Constrains_MF} are violated when
\begin{enumerate}[I]
\item {\it a feature $p\in\lF$ is not assigned to any feature.}\\
That means the MCMF $\mathbf{F}$ used to construct $\M_\labelvars$ does not saturate $\graph^*$ and this is a contradiction to our assumption that $\mathbf{F}$ is a MCMF. Notice edges $(s,n_p)$ for $p\in\lF$ must be saturated as in the worst case scenario $p$ will be matched to another feature through the outlier model for a fixed cost penalty $T$.
\item {\it a feature $p\in\lF$ is assigned to more than one feature in $\rF$, e.g.~$q_1$ and $q_2$.}\\
If there exist two models $h$ and $\ell$ such that $x_{pq_1h}\!\!=\!\!1$ and $x_{pq_2\ell}\!\!=\!\!1$ then $\mathbf{F}(n_{ph},n_{q_1h})\!\!=\!\!1$ and $\mathbf{F}(n_{p\ell},n_{q_2\ell})\!\!=\!\!1$ must be true by construction of $\M_\labelvars$. By construction of $\graph^*$, $n_{ph}$ and $n_{p\ell}$ acquire their flow from $n_p$, and $n_p$ could only push out one unit of flow. Therefore, for $\mathbf{F}(n_{ph},n_{q_1h})=1$ and $\mathbf{F}(n_{p\ell},n_{q_2\ell})=1$ to be true $n_p$ must push two units of flow and that contradicts our assumption that $\mathbf{F}$ is a valid flow over $\graph^*$.
\item {\it a feature $q\in\rF$ is assigned to a zero or more than one feature in $\lF$.}\\
We could show that scenario could not happen for $\M_\labelvars$ by reversing the roles of $p$ and $q$ in I and II.
\item {\it a matched pair of features $p$ and $q$ are assigned to more than one model. e.g.~$h_1$ and $h_2$.}\\
If $x_{pqh_1}\!\!\!=\!\!1$ and $x_{pqh_2}\!\!\!=\!\!1$ then $\mathbf{F}(n_{ph_1},n_{qh_1})\!\!=\!\!1$ and $\mathbf{F}(n_{ph_2},n_{qh_2})\!\!=\!\!1$ must be true by construction of $\M_\labelvars$. By construction of $\graph^*$, $n_{ph_1}$ and $n_{ph_2}$ acquire two units flow from $n_p$ while $n_p$ could only push out one unit of flow. Therefore, for $\mathbf{F}(n_{ph_1},n_{qh_1})\!\!=\!\!1$ and $\mathbf{F}(n_{ph_2},n_{qh_2})\!\!=\!\!1$ to be true $n_p$ must push our two units of flow and that contradicts our assumption that $\mathbf{F}$ is a valid flow over $\graph^*$.
\end{enumerate}

Finally, if such a solution $\M_\labelvars$ exist then $\M^*_\labelvars$ is not optimal as $k^*$ will be bigger than $k$ and that contradicts our main assumption that $\M^*_\labelvars$ an optimal GAP solution, i.e.~$k^*$ is the lowest possible objective value.

Assume that $\mathbf{F^*}$ is a valid MCMF over $\graph^*$ with $cost(\mathbf{F^*})=k^*$. If there exists an feasible solution $\M_\labelvars$ for which the objective value is $k<k^*$ then we can construct a valid MCMF $\mathbf{F}$ where
$$\begin{matrix*}[l]
           &\mathbf{F}(s,n_p)&=1 & \forall p\in \lF \\
           &\mathbf{F}(n_p,n_{ph})&=\begin{cases} 1 & \exists q \in \rF \text{ where } x_{pqh}=1 \\ 0 &\text{otherwise}\end{cases} &\forall p\in \lF,h\in \labelset\\
           &\mathbf{F}(n_{ph},n_{qh})&=x_{pqh} & \forall p\in \lF,q\in \rF, h\in \labelset\\
           &\mathbf{F}(n_{qh},n_q)&=\begin{cases} 1 & \exists p \in \lF \text{ where } x_{pqh}=1\\ 0 &\text{otherwise}\end{cases} &\forall q\in \rF,h\in \labelset \\
           &\mathbf{F}(n_q,t)&=1 & \forall q\in \rF.
\end{matrix*}$$
Using Corollary~\ref{cr:F_GAP_cost} we can deduce that the $cost(\mathbf{F})=k$.

No we will prove that the constructed flow $\mathbf{F}$ is a valid MCMF. A flow is considered valid if it satisfies the capacity and conservation of flow constraints over $\graph^*$.
$\mathbf{F}$ satisfies the capacity constraints by construction of $\mathbf{F}$---the flow through any edge is either 1 or 0 while all edge capacities are 1. Furthermore, $\mathbf{F}$ was constructed in a way that preserves the flow with in $\graph^*$. That is, if there is a flow through edge $(n_{ph},n_{qh})$ we create a flow from $s$ to $n_{ph}$ and from $n_{qh}$ to $t$ along the following paths $\{s,n_p,n_{ph}\}$ and $\{n_{qh},n_q,t\}$, respectively. Therefore, the conservation flow is preserved at $n_{ph}$,$n_{qh}$,$n_p$ and $n_q$.
Notice that there can not exist $n_{ph_1}$ and $n_{ph_2}$ both creating along $\{s,n_p,ph_1\}$ and $\{s,n_p,ph_2\}$ as in this case $\M_\labelvars$ will unfeasible which is a contradiction.
Moreover, the amount of flow going in to $\graph^*$ through $s$ is $|\lF|$ and the amount of flow going out of $\graph^*$ through $t$ is $|\rF|$ and since $\M_\labelvars$ is a feasible GAP solution, by definition, then $|\lF|$ must be equal to $|\rF|$. Thus, the conservation of flow constraint is preserved at $s$ and $t$, and $\graph^*$ is saturated. Thus, the constructed flow $\mathbf{F}$ is a valid MCMF flow.

Finally, if such a solution $\mathbf{F}$ exist then $\mathbf{F^*}$ is not MCMF as $cost(\mathbf{F})<cost(\mathbf{F^*})$ and that contradicts our main assumption that $\mathbf{F^*}$ is a MCMF over $\graph^*$.
\end{proof}
\begin{corollary}
\label{cr:F_GAP_cost}
For a valid $\mathbf{F}$ over $\graph^*$ and a GAP solution $\M_\labelvars$ where $\M_\labelvars=\{x_{pqh}=\mathbf{F}(n_{ph},n_{qh}) \;|\;  p\in \lF,q\in \rF, h\in \labelset\}$, the objective value of the GAP solution  $\M_\labelvars$ is equal to $cost(\mathbf{F})$.
\end{corollary}
\begin{proof}
\begin{flalign*}
cost(\mathbf{F})=&\underset{(v,w)\in \edges}{\sum} c(v,w)\cdot\mathbf{F}(v,w) &&\text{\small by definition of flow cost}\\
=&\!\!\!\!\underset{(n_{ph},n_{qh})\in \edges}{\sum} \!\!\!\!c(n_{ph},n_{qh})\cdot\mathbf{F}(n_{ph},n_{qh}) &&\text{\small by construction of $\graph^*$, other edge costs are 0}\\
=&\!\!\!\!\sum_{(n_{ph},n_{qh})\in \edges}\!\!\!\!\!\!D_{pq}(\param_h)\cdot\mathbf{F}(n_{ph},n_{qh}) &&\text{\small by definition of $c$ over $\edges$}\\
=&\;\;\sum_{\substack{p \in \lF\\ q \in \rF\\h \in \labelset}} D_{pq}(\param_h)\cdot x_{pqh} &&\text{\small by condition in Corollary~\ref{cr:F_GAP_cost}.}
\end{flalign*}
\end{proof}

\section{Total Modularity Proof}
\label{ap:TMproof}
Our generalization of the assignment problem for solving the matching problem over a set of models $\labelset$ such that $|\labelset|\geq 1$ could be formulated as integer linear program \begin{eqnarray}
\mathbf{GAP:}\quad \underset{\M_\labelvars}{\arg \min}\quad  & \underset{h\in \labelset}{\sum}\; \underset{p \in \lF}{\sum} \; \underset{q \in \rF}{\sum}\unary_{pq}(\param_h) x_{pqh}& \nonumber \\
\textrm{s.t.}\quad   &  \underset{h\in \labelset}{\sum}\;\underset{ p \in \lF}{\sum} x_{pqh} =1  & \forall q \in \rF \label{eq:tmproof_first_121Constraint}\\
                     &  \underset{h\in \labelset}{\sum}\;\underset{ q \in \rF}{\sum} x_{pqh} =1  & \forall p \in \lF \label{eq:tmproof_second_121Constraint}\\
                     & x_{pqh} \in \{0,1\} & \forall h \in \labelset ,\; p \in \lF ,\; q \in \rF. \nonumber
\end{eqnarray}
The rest of this section proves that GAP is an {\em integral linear program}, that is, its LP relaxation is guaranteed to have an integer solution.

Let us denote the coefficient matrix and the right hand side vector of equations~\eqref{eq:tmproof_first_121Constraint} and~\eqref{eq:tmproof_second_121Constraint} by $A$ and $b$, respectively. It is known \cite{ILP:alex} that a linear program with constraints $Ax=b$ is {\em integral} for any objective function as long as $b$ is integer, which is true in our case, and matrix $A$ is {\em totally unimodular}. It remains to prove that $A$ is totally unimodular.

\begin{lemma}
\label{lm:TU}
Coefficient matrix $A$ of GAP's linear constraints is totally unimodular.
\end{lemma}

\begin{proof}
The coefficient matrix $A$ of GAP has a special structure that facilities its proof of total unimodularity. Let us assume, without loss of generality, that the number of features on the left and right images is $n$. Then the coefficients matrix in case $\labelset=\{h\}$ could be written as follows
\begin{equation*}
A^{'}=
\begin{array}{cc}
\begin{array}{ccc}
     \begin{array}{cc}  x_{11h} & x_{1nh} \end{array} &
     \begin{array}{c}   \;\;\; \hdots  \end{array} &
     \begin{array}{cc} \quad x_{n1h}&  x_{nnh} \end{array}
   \end{array} & \\
\left(
      \begin{array}{c|c|c}
              \begin{array}{c c c} 1 & \hdots& 1  \\ &        & \\   &        &   \end{array} &
              \begin{array}{c c c}   &        &   \\ & \hdots & \\   &        &   \end{array} &
              \begin{array}{c c c}   &        &   \\ &        & \\ 1 & \hdots & 1 \end{array} \\     \hline
              \begin{array}{c c c} 1 &        &   \\ & \ddots & \\   &        & 1 \end{array} &
              \begin{array}{c c c} 1 &        &   \\ & \ddots & \\   &        & 1 \end{array} &
              \begin{array}{c c c} 1 &        &   \\ & \ddots & \\   &        & 1 \end{array} \\
      \end{array}
\right) &
\begin{array}{c}
\left.\begin{aligned}        \\        \\        \\ \end{aligned} \right\} \text{eqs }\eqref{eq:tmproof_first_121Constraint}\\
\left.\begin{aligned}        \\        \\        \\ \end{aligned} \right\} \text{eqs }\eqref{eq:tmproof_second_121Constraint}
\end{array}\\
\end{array}.
\end{equation*}
In case $|\labelset|>1$ then coefficients matrix $A$ could be written as follows
\begin{equation*}
A=\bordermatrix{~&1& & 2 && \hdots && |\labelset| \cr
&A^{'}&| &A^{'}| &&\hdots &| &A^{'} \cr}.
\end{equation*}

Heller and Tompkins \cite{heller56} showed that in order to prove that $A$ is totally unimodular it is sufficient to prove that the following three conditions are satisfied by the coefficient matrix:
\begin{enumerate}[I]
\item {\em Every entry of the coefficient matrix is either 0, +1, or -1}.

This condition is satisfied for $A$ by construction, see equations~\eqref{eq:tmproof_first_121Constraint} and~\eqref{eq:tmproof_second_121Constraint}.
\item {\em Every column of the coefficient matrix contains at most two non-zero entries}.

Each column in $A$ corresponds to a unique decision variable, for example $x_{pqh}$.
 Note that  variable $x_{pqh}$ appears only once in linear equations~\eqref{eq:tmproof_first_121Constraint} and once in linear equations~\eqref{eq:tmproof_second_121Constraint}. Therefore, variable  $x_{pqh}$ appears twice in $A$. That is, the column corresponding to $x_{pqh}$ has exactly two non-zero entries.

\item {\em There exists a two set partitioning, say $I_1$ and $I_2$, for the rows of the coefficients matrix such that if two non-zero entries in any column have the same sign then these two rows are in different sets. And, if the non-zero entries have different signs then these two rows belong to the same set.}

Notice that $A^{'}$ satisfies condition III by setting $I_1$ and $I_2$ to the rows of~\eqref{eq:tmproof_first_121Constraint} and~\eqref{eq:tmproof_second_121Constraint}, respectively. Also, the coefficients matrix $A$ in case of more than one model is simply the horizontal concatenation of the coefficients matrix $A^{'}$ to itself $|\labelset|$ times. Thus, the constrains added over the two disjoint sets $I_1$ and $I_2$, that satisfy condition III over $A^{'}$, by repeating its columns are redundant. Finally, condition III will be satisfied by $A$ by the same row partitioning that would satisfy condition III for $A^{'}$.
\end{enumerate}
\end{proof}
\bibliographystyle{ieeetr}
\bibliography{ref}
\end{document}